\documentclass[12pt]{article}
\usepackage{a4wide}
\usepackage{latexsym,amssymb,amsmath,epsfig,amsfonts}
\usepackage{algorithm}
\usepackage[noend]{algpseudocode}
\usepackage[title]{appendix}
\usepackage{eucal}
\usepackage{multirow,multicol,array,bm}
\usepackage{mathtools}
\usepackage{tabularx, booktabs, makecell, caption}
\usepackage{siunitx}
\usepackage[utf8]{inputenc}
\usepackage{comment}
\usepackage{bbm}
\usepackage{graphicx}
\usepackage{float}

\usepackage[backend=bibtex,style=numeric]{biblatex}
\usepackage[skip=0.333\baselineskip]{subcaption}
\usepackage[font=normal,labelfont=bf]{caption}

\usepackage{amsthm}
\usepackage{color}

\usepackage[colorlinks=true,allcolors=blue]{hyperref}

\providecommand{\keywords}[1]{\textbf{\textit{Keywords---}} #1}
\numberwithin{table}{section}
\numberwithin{equation}{section}
\newtheorem{theorem}{Theorem}

\newtheorem{proposition}[theorem]{Proposition}

\newtheorem{definition}{Definition}
\begin{document}
\title{A time-series classification framework for individual-level absenteeism prediction under severe class imbalance}

\author{Kwong Ho Li\thanks{School of Mathematical Sciences, Adelaide University, SA 5000, Australia, ronald.li@adelaide.edu.au.} \and Matthew Roughan\thanks{School of Mathematical Sciences, Adelaide University, SA 5000, Australia, matthew.roughan@adelaide.edu.au.}\and Wathsala Karunarathne\thanks{School of Mathematical Sciences, Adelaide University, SA 5000, Australia, wathsala.karunarathne@adelaide.edu.au.}}

\maketitle

\begin{abstract}

   Staff absenteeism imposes substantial operational costs in high-demand work
   environments such as healthcare, emergency services, meat processing, construction,
   and courier and delivery services, where proactive workforce planning depends on
   reliable individual-level absence prediction. Existing regression and classification
   approaches share a structural limitation; they map features observed at time $t$ to
   labels at the same time $t$, reproducing already-realised outcomes rather than
   predicting future events, and discard the sequential behavioural structure inherent in
   individual attendance histories. We propose a Time Series Classification (TSC)
   framework that separates historical attendance sequences from future absence labels,
   enabling genuinely proactive prediction. Due to the lack of public longitudinal
   attendance data, we construct a reproducible simulated dataset calibrated to the UCI
   dataset. We analyse Binary Focal Loss (BFL) and Geometric Mean (G-Mean) loss
   under severe class imbalance using only the imbalance ratio $\rho$. For BFL, the initial
   gradient ratio is $\rho\alpha/(1-\alpha)$, implying the balanced weight
   $\alpha = 1/(1+\rho) \approx 0.023$. Experiments show that performance is governed
   mainly by $\alpha$, with BFL ($\alpha,\gamma=0$) achieving specificity $0.813$
   and balanced accuracy $0.888$, comparable to G-Mean. Unlike BFL, G-Mean adapts
   automatically without parameter calibration. Among three deep learning architectures
   evaluated, Long Short-Term Memory (LSTM), Convolutional Neural Network (CNN),
   and the hybrid LSTM-Fully Convolutional Network (LSTM-FCN), the LSTM-FCN
   delivers strong precision and specificity. Stable performance is obtained with batch
   sizes $\geq 64$ and window sizes between $40$--$80$ days, yielding balanced accuracy
   of approximately $80\%$ on held-out test data. These results provide practical
   guidance for severely imbalanced binary classification in workforce management,
   healthcare, and fraud detection.
\end{abstract}


\section{Introduction}

Effective workforce management in high-demand operational environments requires not only reactive responses to staff absences, but proactive knowledge about which employees are likely to be absent and when. Organisations that fail to anticipate absences face operational disruptions and reduced service efficiency from understaffing; those that overcompensate by maintaining excessive staffing buffers incur unnecessarily high wage costs. This tension is particularly acute in environments characterised by significant physical, psychological, and emotional demands, including healthcare and emergency services~\cite{bardhan2023psychosocial, lawn2020effects}, construction~\cite{gomez2023stress}, meat processing industries~\cite{macnair2023should}, and courier and delivery services, where workers are subject to physically demanding conditions, time pressure, fatigue accumulation, and sustained stress~\cite{slade2023psychological}. Reliable individual-level absence prediction is therefore a core knowledge requirement for data-driven workforce decision support. In this study, we define an employee as absent when they fail to appear for an allocated shift without prior notification to their line manager. 

Despite the evident operational need, the knowledge representation underlying existing machine learning approaches to absenteeism prediction remains inadequate for proactive decision support. Existing methods primarily employ regression, classification, or time-series forecasting frameworks~\parencite{llamasblazquez2025predicting,Popa2025An,Asghar2021Trends}, treating absenteeism as a duration, rate, or aggregated class variable derived from employee characteristics or aggregate temporal indicators~\parencite{llamasblazquez2025predicting,piciga2024predicting,Popa2025An}. These formulations are limited by their reliance on same time mappings; features $X_t$ observed at time $t$ are mapped to labels $y_t$ at that same time $t$, so that by the time input features are available, the attendance outcome has already occurred.  Rather than supporting proactive workforce decisions, such models learn to reproduce already-realised outcomes. A further limitation is that these approaches aggregate individual behaviour into population-level summaries, discarding the sequential patterns within each employee's attendance history that may carry predictive information, such as accumulating fatigue, recurring health episodes, or deteriorating work-life balance.

Time Series Classification (TSC) addresses both limitations within a unified framework. As a supervised learning paradigm, TSC learns a mapping from historical sequences of observations to future class labels, preserving the temporal ordering between predictors and targets and enabling genuinely forward-looking inference. While deep learning-based TSC has demonstrated strong performance in healthcare, power systems, and finance~\parencite{bagnall2016greattimeseriesclassification,Xie2024Prototype,Hirnschall2025Semi}, its application to sequential individual attendance behaviour has not previously been explored. Operationalising TSC for absenteeism prediction introduces two additional methodological challenges. First, individual-level attendance records are sensitive longitudinal data that are not publicly available, precluding direct empirical evaluation on real organisational datasets. Second, attendance records exhibit severe class imbalance; in this paper, attendance (present) is the positive class ($y = 1$) and constitutes approximately $97.69\%$ of records, while absence is the negative class ($y = 0$) at approximately $2.31\%$, yielding an imbalance ratio $\rho = N^{+}/N^{-} \approx 42$. Standard loss functions tend to bias models towards the majority class, resulting in poor specificity (the ability to correctly identify genuinely absent employees), which is the operationally critical failure mode in workforce planning.

The choice of loss function under severe imbalance has received limited theoretical analysis within the TSC literature. While Binary Focal Loss (BFL)~\parencite{lin2017focal} is widely used for imbalanced classification, and Geometric Mean (G-Mean) loss has been proposed as an alternative~\parencite{raj2016towards}, their gradient dynamics remain under-explored. Specifically, how these functions behave under the majority-positive convention of attendance prediction and when BFL’s focusing mechanism fails, has not yet been analytically characterised.

This paper proposes a TSC framework for individual-level absenteeism prediction, presents experimental results comparing BFL and G-Mean loss, and provides a theoretical analysis of the observed performance gap. To address the absence of publicly available longitudinal attendance data, we construct a reproducible simulated dataset rigorously calibrated to the distributional properties of the UCI Machine Learning Repository \emph{Absenteeism at Work} benchmark~\parencite{absenteeism_at_work_445}, establishing a
controlled experimental environment for methodological evaluation. We evaluate three established deep learning architectures: Long Short-Term Memory (LSTM), Convolutional Neural Network (CNN), and the hybrid LSTM-Fully Convolutional Network (LSTM-FCN), and two imbalance-aware loss functions, G-Mean loss and BFL, across a range of hyperparameter configurations. The resulting framework produces categorical absence predictions from historical attendance sequences, requiring no further transformation for direct use in workforce scheduling systems.

This paper makes four principal contributions. First, we propose a TSC formulation for individual-level absenteeism prediction as a sequential knowledge problem. We reformulate staff absenteeism prediction as a time series classification task at the individual level, explicitly encoding the temporal separation between historical attendance sequences and future absence labels. This formulation enables the extraction of sequential behavioural knowledge from attendance histories for proactive prediction, in contrast to existing approaches that model already-realised outcomes or discard individual-level temporal structure. Second, we provide a theoretical analysis of loss function gradient dynamics under severe class imbalance, grounded in experimental findings. We demonstrate empirically that BFL performance is governed by $\alpha$, with the analytically derived $\alpha^{*} = 1/(1+\rho) \approx 0.023$ ($\rho = 42.31$, from the training data) producing the best BFL results. We derive that the BFL gradient ratio at initialisation is $\rho\alpha/(1-\alpha)$, independent of $\gamma$, and prove that G-Mean loss exhibits a self-correcting gradient structure invariant to $\rho$. A systematic sweep across $\alpha \in \{\alpha^{*}, 0.05, 0.1, 0.2\}$ and $\gamma \in \{0, 1, 2, 5\}$ confirms that specificity increases monotonically as $\alpha \to \alpha^{*}$, that $\gamma = 0$ is optimal at $\alpha^{*}$, and that BFL with $\alpha^{*}$, $\gamma = 0$ achieves specificity $0.813$ and balanced accuracy $0.888$, competitive with G-Mean (specificity $0.844$) which requires no calibration. These results provide a principled basis for loss function selection beyond the absenteeism domain. Third, we present a systematic evaluation of architectures and hyperparameters for attendance sequence modelling. We establish that the LSTM-FCN architecture achieves the most consistent precision and specificity, that model performance stabilises for batch sizes of at least $64$, and that window sizes between $40$ and $80$ days yield the most balanced performance for a 5-day prediction horizon. These findings constitute reusable design knowledge for practitioners applying TSC to imbalanced sequential prediction tasks. Fourth, we develop a calibrated simulation framework for longitudinal attendance research. We develop and publicly release a simulation methodology that generates synthetic individual-level attendance sequences calibrated to empirical distributional properties, addressing the absence of publicly available longitudinal data and providing a reproducible baseline for future comparative evaluation.

The remainder of this paper is organised as follows. Section~\ref{sec:literature_review} reviews related work on machine learning approaches to absenteeism prediction and deep learning-based TSC. Section~\ref{sec:TSC} describes the proposed TSC framework, including the data simulation procedure, model architectures, loss functions, and experimental design. Section~\ref{sec:results} presents experimental results and, following the empirical loss function comparison, provides a theoretical analysis of the observed specificity gap and derives practical guidance for applying TSC to absenteeism prediction. Section~\ref{sec:conclusion} concludes with directions for future research.

\section{Related work}\label{sec:literature_review}

\subsection{Current machine learning approaches}\label{sec:current_ML_approaches}

In studies applying machine learning to the prediction of staff absenteeism, researchers typically formulate the problem as one of three machine learning tasks: regression, classification, or time-series forecasting.

\subsubsection{Regression} 

Regression-based absenteeism studies have employed a range of standard models, including Linear Regression, Support Vector Regression (SVR), Decision Tree and Random Forest Regressor, K-Nearest Neighbours (KNN), Gradient Boosted Regression Trees (GBRT), and Neural Networks \parencite{xames2025socially, tewari2020predictive, piciga2024predicting, llamasblazquez2025predicting}. \textcite{llamasblazquez2025predicting} employed a regression-based approach using the \emph{Absenteeism at Work} dataset, which consists of three years of data from a Brazilian courier company obtained from the UCI Machine Learning Repository~\parencite{absenteeism_at_work_445}. The dataset comprises $740$ absence records described by a range of demographic, clinical, and occupational characteristics. The study focused on predicting absence duration using features such as age, body mass index (BMI), distance from residence to workplace, transportation expense, number of children, education level (coded from $1$ to $4$), length of service, binary indicators of social drinking and smoking behaviour, day of the week, month, and seasonal factors. The best-performing regression models reported in this study achieve a Coefficient of Determination ($R^2$) of $0.13$. Regression-based studies commonly assess model performance using the $R^2$, Mean Squared Error ($MSE$), Root Mean Squared Error ($RMSE$), and Mean Absolute Error ($MAE$) \parencite{xames2025socially, tewari2020predictive, piciga2024predicting, llamasblazquez2025predicting}. 

\subsubsection{Classification}

Classification is the predominant formulation for staff absenteeism prediction. A wide range of machine learning models have been applied, including Decision Trees, Gradient Boosted Trees (GBRT), Random Forests, CatBoost, Bayesian Networks, Naïve Bayes, Support Vector Machines (SVMs), Multi-layer Perceptrons (MLPs), Multinomial Logistic Regression (MLR), and Transformer-based architectures. Tree-based methods are among the most commonly used approaches for this task \parencite{Wahid2019Predicting, Nath2022Incorporating, Raman2020A, Popa2025An, alzubi_classification_2024, Lumintu2025Supporting}. The \emph{Absenteeism at Work} dataset from the UCI Machine Learning Repository is widely used for classification-based studies ~\parencite{absenteeism_at_work_445}. Several labelling strategies have been proposed. \textcite{Wahid2019Predicting} transformed absenteeism hours into four categorical classes: \emph{not absent}, \emph{hours}, \emph{days}, and \emph{weeks}. \textcite{Nath2022Incorporating} grouped absenteeism into three ordinal categories (A+, B+, and C+), while \textcite{Raman2020A} adopted a binary classification scheme indicating absence or non-absence. Both \textcite{Wahid2019Predicting, Raman2020A} utilised the full feature set excluding staff identifiers, whereas \textcite{Nath2022Incorporating} applied feature selection to construct reduced input sets for model training. The best-performing models reported in these studies achieve classification accuracies of approximately $80$--$90\%$~\parencite{Wahid2019Predicting, Nath2022Incorporating}.

\begin{table}[h]
	\centering
	\begin{tabular}{c|cc}
		\hline
		& \textbf{Predicted Positive} & \textbf{Predicted Negative} \\
		\hline
		\textbf{Actual Positive} & True Positive (TP) & False Negative (FN) \\
		\textbf{Actual Negative} & False Positive (FP) & True Negative (TN) \\
		\hline
	\end{tabular}
	\caption{Confusion matrix. In this paper, positive ($y=1$) denotes attendance (present) and negative ($y=0$) denotes absence.}
	\label{tab:confusion_matrix}
\end{table}

Classification performance is commonly assessed using a set of complementary evaluation metrics derived from the confusion matrix (Table~\ref{tab:confusion_matrix}). These metrics capture different aspects of predictive accuracy and error characteristics. Commonly reported metrics include Accuracy, Precision, Recall, and $F1$ Score \parencite{Wahid2019Predicting, Nath2022Incorporating, Raman2020A, Popa2025An, alzubi_classification_2024, Lumintu2025Supporting}.

\subsubsection{Time series forecasting}

Compared with regression and classification approaches, time series forecasting has received relatively limited attention in the staff absenteeism literature. Existing studies primarily model temporal dynamics in aggregated absence rates, often incorporating exogenous covariates. A limited yet diverse set of models has been applied, ranging from classical statistical methods to hybrid neural frameworks. These include Autoregressive Integrated Moving Average (ARIMA), Seasonal ARIMA (SARIMA), and NeuralProphet. ARIMA and SARIMA capture temporal dependencies and seasonality through autoregressive and moving average components, while NeuralProphet extends this by integrating neural network-based autoregressive structures with classical trend and seasonality components, enabling the modelling of non-linear dynamics and exogenous covariates \parencite{Asghar2021Trends, Lindell2023Processing, lyszczarz_excess_2025}. In time series forecasting studies, the Akaike Information Criterion (AIC) is commonly used for model selection when predicting future trends. When the focus is on prediction accuracy, regression metrics such as MSE and MAE are typically employed \parencite{Asghar2021Trends, Lindell2023Processing, lyszczarz_excess_2025}.

\subsection{Time series classification}

Deep learning-based TSC has achieved notable success across multiple sectors. In healthcare, recurrent architectures such as LSTM networks have been used to classify multivariate electronic health records and detect abnormal electrocardiogram (ECG) signals~\parencite{Xie2024Prototype}. In power systems, these approaches enable accurate detection and classification of electrical faults, improving grid reliability~\parencite{Jiriwibhakorn2025Time}. Similarly, in finance, deep learning models support efficient and uncertainty-aware classification for tasks such as credit card fraud detection, enhancing risk management~\parencite{Hirnschall2025Semi}. Following the formalisation presented by \textcite{ismail_fawaz_deep_2019}, a dataset for TSC can be defined as follows: 

\begin{definition} 

A univariate time series $X = [x_1, x_2, \dots, x_T]$ is an ordered set of real values. The length of $X$ is equal to the number of real values $T$.
\end{definition}
	
\begin{definition} 

An $M$-dimensional Multivariate Time Series (MTS), $X = [X^1, X^2, \dots, X^M]$ consists of $M$ different univariate time series with $X^i \in \mathbb{R}^T$.

\end{definition}
	
\begin{definition}  A dataset $D = \{(X_1, Y_1), (X_2, Y_2), \dots, (X_N, Y_N)\}$ is a collection of pairs $(X_i, Y_i)$ where $X_i$ could either be a univariate or multivariate time series with $Y_i$ as its corresponding one-hot label vector. For a dataset containing $K$ classes, the one-hot label vector $Y_i$ is a vector of length $K$ where each element $j \in [1, K]$ is equal to $1$ if the class of $X_i$ is $j$ and $0$ otherwise.

\end{definition}

The objective of TSC is to train a classifier on dataset $D$ that maps input time series to a probability distribution over the possible class labels.

Among the deep learning architectures proposed for TSC, recurrent models such as the Gated Recurrent Unit (GRU)~\parencite{Tan2021Time} and LSTM ~\parencite{Karim2018LSTM}, convolutional approaches such as Residual Networks (ResNets), InceptionTime~\parencite{Wang2017Time, Ismail2020InceptionTime} and the Fully Convolutional Network (CNN/FCN)~\parencite{Wang2017Time}, hybrid architectures such as
LSTM-FCN~\parencite{Karim2018LSTM} and its multivariate extension MLSTM-FCN~\parencite{Fazle2019Multivariate}, and Transformer-based frameworks~\parencite{Zerveas2021Transformer} have all demonstrated competitive performance on benchmark datasets. This study focuses on three architectures selected for their established performance on univariate classification tasks and computational tractability: LSTM, CNN/FCN, and the hybrid LSTM-FCN. 

\section{Methods} \label{sec:TSC}

The TSC formulation introduced above resolves the same-time mapping limitation identified in Section~\ref{sec:literature_review}: rather than learning a mapping from input features ($X$) to target labels ($y$) observed within the same time context, TSC explicitly separates the input sequence from the future label it predicts. The term \emph{time series} refers here to the temporal structure of the data, where input observations precede the corresponding target labels in time, so that the model learns from historical attendance patterns to predict future attendance outcomes rather than reproduce already-realised ones. \emph{Classification} defines the learning objective as the prediction of discrete categorical labels, whilst deep learning provides neural architectures capable of capturing complex temporal dependencies within sequential data.

Framing the problem this way enables the model to learn relationships between past attendance behaviour and future absence outcomes, supporting proactive prediction of upcoming absences rather than the retrospective characterisation produced by the same-time formulations reviewed above. This is substantially more valuable for workforce planning and resource allocation, and producing categorical outputs aligns directly with operational decision-making, requiring no further transformation.

As individual-level longitudinal attendance records are not publicly available,  we constructed a simulated dataset based on the \emph{Absenteeism at Work} dataset~\parencite{absenteeism_at_work_445}. This simulation is designed to approximate daily attendance sequences and to support a demonstrative implementation of deep learning-based TSC. In the following sections, we describe the data simulation procedure, model architectures, loss functions, and training settings. As results are derived from simulated data, the evaluation centres on methodological feasibility and training behaviour, rather than absolute predictive accuracy.

\subsection{Data simulation}

We simulate attendance records based on the UCI Machine Learning Repository dataset~\parencite{absenteeism_at_work_445}. The dataset contains $740$ absence records from $36$ employees collected over 36 months (July 2007 to July 2010) at a courier company in Brazil. Each record contains demographic, occupational, and health-related characteristics, along with absence duration measured in hours. To generate a time series dataset suitable for TSC, we implemented a two-stage data expansion and synthesis procedure. First, the monthly-level data was transformed into daily attendance records for the original $36$ employees. Second, these records were used as seed data to generate a larger synthetic workforce through distribution-preserving sampling.

\subsubsection{Stage 1: Expansion from monthly to daily resolution}

The expansion procedure converted aggregate monthly absence records into binary daily attendance sequences. Records with zero absence hours were removed, as they represent no event markers rather than actual attendance observations. For each remaining absence event, the number of absent work days was computed as the ceiling of absence hours divided by 8, assuming a standard 8-hour work day.

To resolve temporal ambiguity in the original dataset, which recorded only the month of absence without specifying the year, we applied a sequential assignment strategy. Each absence record specifies a day of week indicator, from which  all matching day of week within the assigned month and year were identified. The required number of absence days was then randomly sampled from these candidates without replacement. A complete calendar spanning all days from $1^{st}$ of July 2007 to $31^{st}$ July 2010 was generated, resulting in $1,127$ consecutive days. Binary attendance labels were assigned to each employee-day combination, with absence marked as $0$ and all remaining days marked as $1$ (present).

\subsubsection{Stage 2: Synthetic employee generation}

To demonstrate the TSC framework at scale, we expanded the $36$ employee seed dataset by generating $964$ additional synthetic employees through distribution-preserving sampling, resulting in a total of $1,000$ employees. While larger datasets would be preferable for production 
deployment, this size enables methodological evaluation while maintaining 
computational tractability. For each employee characteristic, we used empirical distributions from the seed data; continuous variables (age, BMI, service time, height, weight, transportation expense, distance from residence to work, hit  target) were modelled using normal distributions fitted to sample means and standard deviations. Categorical variables (education level, number of children, social drinker status, social smoker status, number of pets) were modelled using discrete probability distributions based on observed frequencies.

Individual absence rates across the $36$ seed employees were computed as the proportion of days absent over the observation period, resulting in a mean absence rate of $2.31\%$ and a standard deviation of $2.23\%$. For each synthetic employee, characteristics were sampled from the extracted distributions. An 
individual absence rate was sampled from a normal distribution with mean 
$2.31\%$ and standard deviation $2.23\%$, then clipped to the range $[0, 10\%]$.  Daily absence probabilities were computed by scaling the seed  data's daily absence rates accordingly, and binary attendance labels were generated through independent Bernoulli sampling.

\subsection{Time Series Data Splitting}

\begin{figure}[h!]
	\centering
	\includegraphics[width=0.2\textwidth, angle=90]{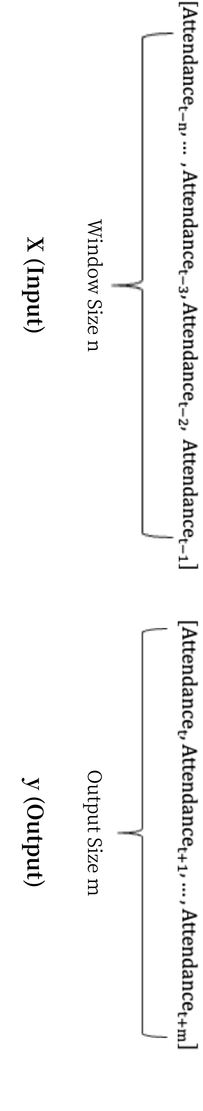}
	\caption{Illustration of the sliding window formulation for time series prediction without staff background information. The past $n$ consecutive days of attendance form the input sequence $X$, and the subsequent $m$ days form the target label $y$.}
	\label{fig:trad_TSC}
\end{figure}

\begin{figure}[h!]
	\centering
	\includegraphics[width=0.2\textwidth, angle=90]{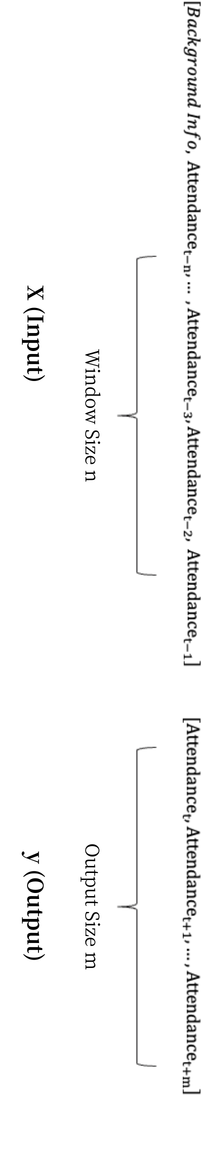}
	\caption{Illustration of the sliding window formulation for time series prediction with staff background information. As in Figure~\ref{fig:trad_TSC}, the input sequence $X$ spans $n$ days and the target label $y$ spans the subsequent $m$ days; here $X$ additionally includes static employee characteristics.}
	\label{fig:hybrid_TSC}
\end{figure}

To construct the TSC dataset, we implemented a custom function to transform daily attendance records into fixed-length time series segments and to partition the data into training, validation, and testing sets. Consistent with the definition of the TSC formulation, the input data must respect temporal ordering of observations to avoid leakage and preserve the sequential structure of the input. Accordingly, samples are generated by applying a sliding window over the attendance records, as illustrated in Figure~\ref{fig:trad_TSC}. Specifically, for each staff member, the past $n$ consecutive days serve as the input sequence $X$ (Definition 1 and 2), and the subsequent $m$ days as the corresponding target label $y$ (Definition 3), where $n$ and $m$ are positive integers. This process is repeated across all valid time windows, dates, and staff members, yielding an aggregated dataset of time series samples suitable for supervised learning. 

In addition to this standard approach, we investigated a hybrid splitting strategy that incorporates static background information into the input sequences. While preserving the temporal structure of the attendance data, this hybrid formulation augments the time series input with additional contextual features, enabling an assessment of how background attributes influence predictive performance. This hybrid splitting strategy is illustrated in Figure~\ref{fig:hybrid_TSC}.

The resulting dataset is partitioned into training, validation, and testing subsets using a 60:20:20 split. To prevent information leakage and ensure a realistic evaluation setting, the split is performed at the staff level based on staff identifiers, such that records from the same individual do not appear across multiple subsets.

\subsection{Data preprocessing}

As an initial preprocessing step, the \texttt{Disciplinary\_failure} feature was removed from the dataset, as it contained no variance and consisted entirely of zero-valued entries. Retaining such a constant feature would not contribute meaningful information to the learning process.

For all numerical attributes, including \texttt{Transportation\_expense}, \texttt{Distance\_from\_\\Residence\_to\_Work}, \texttt{Service\_time}, \texttt{Age}, \texttt{Hit\_target}, \texttt{Weight}, \texttt{Height}, and \texttt{Body\_mass\_index}, feature scaling was applied using standardisation. Specifically, a \texttt{StandardScaler} was fitted on the training data and subsequently applied to the validation and test sets to ensure consistent normalisation while avoiding information leakage.

\subsection{Model architectures }

Three deep learning architectures were evaluated in this study: a pure LSTM model, a pure CNN model, and a hybrid LSTM-FCN model. The LSTM-FCN architecture follows the design proposed by \textcite{Karim2018LSTM}. Tables~\ref{tab:lstm_arch}, \ref{tab:cnn_arch}, and \ref{tab:lstm_fcn_arch} summarise the architectures of the LSTM, CNN, and LSTM-FCN models, respectively. The batch size is denoted by $B$, defined as the number of samples processed concurrently during a single forward pass of the network. Processing data in batches, rather than individually, enhances computational efficiency and contributes to more stable training dynamics. The LSTM model reshapes each input sample of dimension $23$ (corresponding to the number of input features) into a sequence and processes it using a recurrent layer with 64 hidden units, where the final hidden state is used as a feature representation for a fully connected layer to produce a scalar prediction. In contrast, the CNN model reshapes the input into a one-dimensional sequence and applies two convolutional blocks to extract hierarchical features, which are aggregated via global average pooling and mapped to a scalar output through a fully connected layer. The LSTM-FCN model integrates both approaches by combining a convolutional branch, which captures local patterns using a series of 1D Convolutional layers and global pooling, with a recurrent branch that models sequential dependencies; their representations are concatenated and passed to a fully connected layer to generate the final prediction.

\begin{table}[H]
	\centering
	\begin{tabular}{lllll}
		\hline
		\textbf{Component} & \textbf{Layer Type} & \textbf{Input Shape} & \textbf{Output Shape} \\
		\hline
		Input Layer & Pre-processing & $(B, 23)$ & $(B, 1, 23)$ \\
		Recurrent Layer & \texttt{nn.LSTM} & $(B, 1, 23)$ & $(1, B, 64)$  \\
		Feature Extraction & Hidden State Select & $(1, B, 64)$ & $(B, 64)$  \\
		Output Layer & \texttt{nn.Linear} & $(B, 64)$ & $(B, 1)$  \\
		Post-processing & Squeeze & $(B, 1)$ & $(B)$ \\
		\hline
	\end{tabular}
    \caption{Architecture of the LSTM model. $B$ denotes batch size; the input dimension $23$ corresponds to the number of employee features.}
    \label{tab:lstm_arch}
\end{table}

\begin{table}[H]
	\centering
	\begin{tabular}{lllll}
		\hline
		\textbf{Component} & \textbf{Layer Type} & \textbf{Input Shape} & \textbf{Output Shape} \\
		\hline
		Input Layer & Pre-processing & $(B, 23)$ & $(B, 1, 23)$  \\
		Block 1 & Conv1D + BN + ReLU & $(B, 1, 23)$ & $(B, 32, 23)$  \\
		Block 2 & Conv1D + BN + ReLU & $(B, 32, 23)$ & $(B, 64, 23)$ \\
		Global Pooling & AdaptiveAvgPool1D & $(B, 64, 23)$ & $(B, 64, 1)$\\
		Flattening & Squeeze & $(B, 64, 1)$ & $(B, 64)$ \\
		Output Layer & \texttt{nn.Linear} & $(B, 64)$ & $(B, 1)$ \\
		Post-processing & Squeeze & $(B, 1)$ & $(B)$ \\
		\hline
	\end{tabular}
    \caption{Architecture of the CNN model. $B$ denotes batch size; BN denotes batch normalisation.}
    \label{tab:cnn_arch}
\end{table}

\begin{table}[H]
	\centering
	\begin{tabular}{lllll}
		\hline
		\textbf{Component} & \textbf{Layer Type} & \textbf{Input Shape} & \textbf{Output Shape}  \\
		\hline
		Input Layer& Pre-processing & $(B, 23)$ & $(B, 1, 23)$ \\
		FCN Branch & 3$\times$ Conv1D Blocks & $(B, 1, 23)$ & $(B, 128, 23)$ \\
		FCN Pooling & AdaptiveAvgPool1D & $(B, 128, 23)$ & $(B, 128)$ \\
		LSTM Branch & Transpose & $(B, 1, 23)$ & $(B, 23, 1)$\\
		LSTM Recurrent & \texttt{nn.LSTM} & $(B, 23, 1)$ & $(1, B, 128)$ \\
		LSTM Dropout & \texttt{nn.Dropout} & $(B, 128)$ & $(B, 128)$  \\
		Fusion Layer & Concatenation & $(B, 128),(B,128)$ & $(B, 256)$\\
		Output Layer & \texttt{nn.Linear} & $(B, 256)$ & $(B, 1)$\\
		Post-processing & Squeeze & $(B, 1)$ & $(B)$ \\
		\hline
	\end{tabular}
    \caption{rchitecture of the LSTM-FCN model. $B$ denotes batch size; the FCN and LSTM branches process the input in parallel and are concatenated before the output layer.}
    \label{tab:lstm_fcn_arch}
\end{table}

\subsection{Loss function}

The attendance labels in the dataset exhibit a high degree of class imbalance, with attendance records substantially outnumbering absence records. This imbalance reflects typical business operations, as sustained organisational functioning requires the majority of staff to be present on most working days, while absences occur relatively infrequently. However, such imbalance poses challenges for supervised learning, as standard loss functions may bias the model toward the majority class. To address this issue, we adopt two imbalance-aware strategies: BFL, which dynamically downweights well-classified majority samples to focus training on difficult minority cases, and G-Mean loss, which optimises the model by jointly maximising the true positive and true negative rates, thereby promoting balanced performance across both classes. We then compare these approaches to determine which yields superior specificity and balanced accuracy under severe class imbalance.

BFL extends Binary Cross-Entropy by introducing a focusing parameter $\gamma$ that down-weights the contribution of well-classified (``easy'') examples, which typically belong to the majority class. This mechanism encourages the model to focus on misclassified or hard examples from the minority class. In addition, a class-balancing factor $\alpha$ can be applied to control the relative importance of positive and negative samples. Together, these modifications prevent the loss from being dominated by abundant easy samples and improve sensitivity to rare but critical events, such as staff absence~\parencite{Ayesha2023A}. The BFL is defined as:
\begin{equation}
\text{FL}(p, y) =
\begin{cases}
-\alpha (1 - p)^{\gamma} \log(p), & \text{if } y = 1, \\
-(1 - \alpha) p^{\gamma} \log(1 - p), & \text{if } y = 0,
\end{cases}
\end{equation}


where $p \in [0,1]$ denotes the predicted probability of the positive class, $y \in \{0,1\}$ is the ground-truth label, $\gamma \geq 0$ is the focusing parameter that reduces the relative loss contribution of well-classified examples, and $\alpha \in [0,1]$ is a class-balancing factor that controls the relative importance of positive and negative samples.

While G-Mean is widely used as an evaluation metric for imbalanced classification, it can also serve as a training objective by directly optimising the geometric mean of True Positive Rate (TPR) and True Negative Rate (TNR). To enable gradient-based optimisation, TPR and TNR are computed using soft predictions from the sigmoid output rather than hard binary labels, yielding a differentiable approximation suitable for back-propagation~\parencite{raj2016towards}. The G-Mean loss is defined as:
\begin{equation}
  \mathcal{L}_{\text{GM}} = 1 - \sqrt{\widehat{\text{TPR}} \times
    \widehat{\text{TNR}}},
  \label{eq:gmean}
\end{equation}
where the soft approximations are:
\begin{equation}
  \widehat{\text{TPR}} = \frac{1}{N^{+}}\sum_{i:\,y_i=1} p_i
  \quad \text{(soft rate of correctly predicting present)},
  \label{eq:soft_tpr}
\end{equation}
\begin{equation}
  \widehat{\text{TNR}} = \frac{1}{N^{-}}\sum_{j:\,y_j=0} (1 - p_j)
  \quad \text{(soft rate of correctly predicting absent)},
  \label{eq:soft_tnr}
\end{equation}
and $N^{+}$, $N^{-}$ denote the number of positive (present) and
negative (absent) samples in the batch, respectively.

\subsection{Evaluation metrics}\label{sec:metrics}

Multiple evaluation metrics were employed to assess model performance, with metric selection guided by the practical implications of prediction errors in an operational setting. In the context of staff absenteeism, two types of misclassification are of particular interest: predicting an employee as present when they are actually absent (false positive), and predicting an employee as absent when they are actually present (false negative). 

From an operational perspective, false positives are more disruptive, as they may lead to insufficient staffing and require immediate managerial intervention to maintain service efficiency. Therefore, minimising false positives is the primary objective of this study. The false positive rate (FPR) is used to quantify this error, and specificity ($1-\text{FPR}$) is adopted as a key metric, where higher specificity indicates better control of false positives. Additionally, since recall is insensitive to false positives and standard accuracy is unreliable under class imbalance; therefore we use balanced accuracy. Accordingly, this study evaluates model performance using precision, recall, specificity, F1-score, and balanced accuracy, defined as follows:

\begin{itemize}

    \item \textbf{Precision:} Quantifies the proportion of predicted positive instances that are truly positive:
    	\begin{equation}
    	\text{Precision} = \frac{TP}{TP + FP}.
    	\end{equation}

    \item \textbf{Recall (Sensitivity or True Positive Rate - TPR):} Measures the proportion of actual positive instances correctly identified by the model:
    	\begin{equation}
    	\text{Recall} = \frac{TP}{TP + FN}.
    	\end{equation}
	
	\item \textbf{Specificity:} Represents the proportion of actual negative instances correctly classified:
	\begin{equation}
	\text{Specificity} = \frac{TN}{TN + FP}.
	\end{equation}

    \item \textbf{F1 Score:} The harmonic mean of precision and recall, providing a balanced measure of classification performance, particularly under class imbalance:
    	\begin{equation}
    	\text{F1} = 2 \times \frac{\text{Precision} \times \text{Recall}}{\text{Precision} + \text{Recall}} = \frac{2TP}{2TP + FP + FN}.
        \end{equation}

    \item \textbf{Balanced Accuracy:} Measures classification performance by equally weighting sensitivity and specificity, making it suitable for imbalanced datasets:
	\begin{equation}
	\text{Balanced Accuracy} = \frac{\text{Sensitivity} + \text{Specificity}}{2}.
	\end{equation}

\end{itemize}

\subsection{Training process}

Model training was conducted under a fixed set of hyper-parameters and optimisation settings. The learning rate was set to $10^{-3}$, and models were trained for a maximum of 300 epochs, as prior TSC research \parencite{taherkhani_deep_2023} suggests convergence typically occurs after more than 200 epochs. To mitigate overfitting, early stopping was applied based on validation performance, with training terminated if no improvement was observed for 10 consecutive epochs. A minimum improvement threshold of $10^{-4}$ in validation loss was used to determine convergence. The model parameters corresponding to the best validation performance were retained. For binary classification, the output of the sigmoid activation function was converted to class labels using a decision threshold of 0.5, providing a balanced decision boundary between classes \parencite{bishop2006prml}.

\subsection{Computational environment}

All experiments were conducted on a workstation running Windows 11 (64-bit), equipped with an Intel Core i9-14900K CPU, an NVIDIA GeForce RTX 4090 GPU, and 64~GB of system memory. Model development and training were performed using Visual Studio Code within a Python 3.10 environment. GPU acceleration was enabled via CUDA version 11.8, and all deep learning models were implemented using PyTorch version 2.2.0.

\section{Results}\label{sec:results}

\subsection{Batch size}\label{sec:batch_size}

\begin{figure}[h!]
	\centering
	\includegraphics[width=0.8\textwidth]{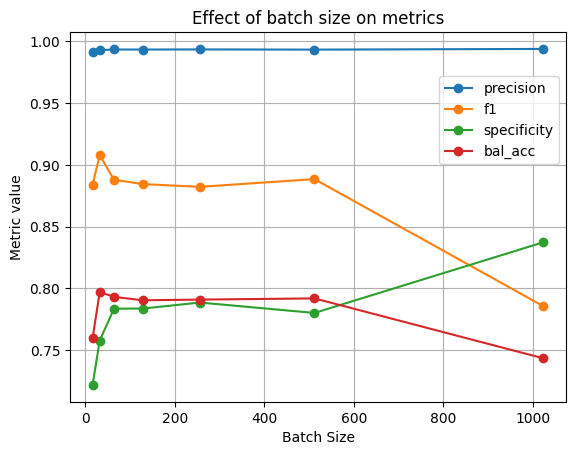}
	\caption{Effect of batch size on validation performance of the LSTM-FCN model trained without background information. Batch sizes range from $16$ to $1{,}024$ in doublings. All runs use window size $80$, output horizon $5$ days, G-Mean loss ($\text{eps} = 10^{-12}$, $\text{use\_log} = \text{True}$), learning rate $10^{-3}$, and early stopping with patience $10$. G-Mean loss has no $\alpha$ or $\gamma$ parameters; no loss function assumptions
    are made in this sweep.}
	\label{fig:batch_comparison}
\end{figure}

Figure~\ref{fig:batch_comparison} illustrates the effect of batch size on the LSTM-FCN model trained without background information. Batch sizes were swept across $\{16, 32, 64, 128, 256, 512, 1{,}024\}$ with all other settings fixed: window size $80$ days (an initial value used here; the window size sweep in Section~\ref{sec:window_size} confirms this choice), output horizon $5$ days, G-Mean loss function, learning rate $10^{-3}$, and early stopping with patience $10$ epochs. G-Mean loss is parameter-free with respect to class weighting, so it has no $\alpha$ or $\gamma$ parameters and no loss function assumptions are made in this sweep.

Both specificity and balanced accuracy are notably low for batch sizes below $64$ but improve substantially thereafter, stabilising from $64$ to $512$. Precision remains relatively stable across all batch sizes, while the F1 score stabilises from a batch size of $64$ onwards. A minimum batch size of $64$ is therefore recommended.

This behaviour is consistent with the self-correcting gradient property of G-Mean loss established in Proposition~\ref{prop:gmean}. At small batch sizes, the expected number of absent-class samples per batch is only $\bar{n}^{-} = B/(1+\rho) \approx B/43.31$, meaning that for $B = 16$, fewer than one absent-class sample is expected on average. When no absent-class samples appear in a batch, $\widehat{\text{TNR}}$ is undefined and the self-correcting mechanism cannot function, causing specificity to collapse. Larger batch sizes increase $\bar{n}^{-}$, stabilising the soft TNR estimate and enabling reliable gradient self-correction. From a computational perspective, larger batch sizes also improve training efficiency by reducing the number of parameter update steps per epoch~\parencite{Samuel2018Don,Priya2018Priya}. A batch size of $1{,}024$ achieves the highest specificity of all batch sizes tested, with an associated reduction in F1 score and balanced accuracy relative to the stable plateau observed between $64$ and $512$. Since specificity is the primary metric of interest in this study, a batch size of $1{,}024$ is selected for all subsequent experiments.

\subsection{Window size} \label{sec:window_size}

\begin{figure}[h!]
	\centering
	\includegraphics[width=1\textwidth]{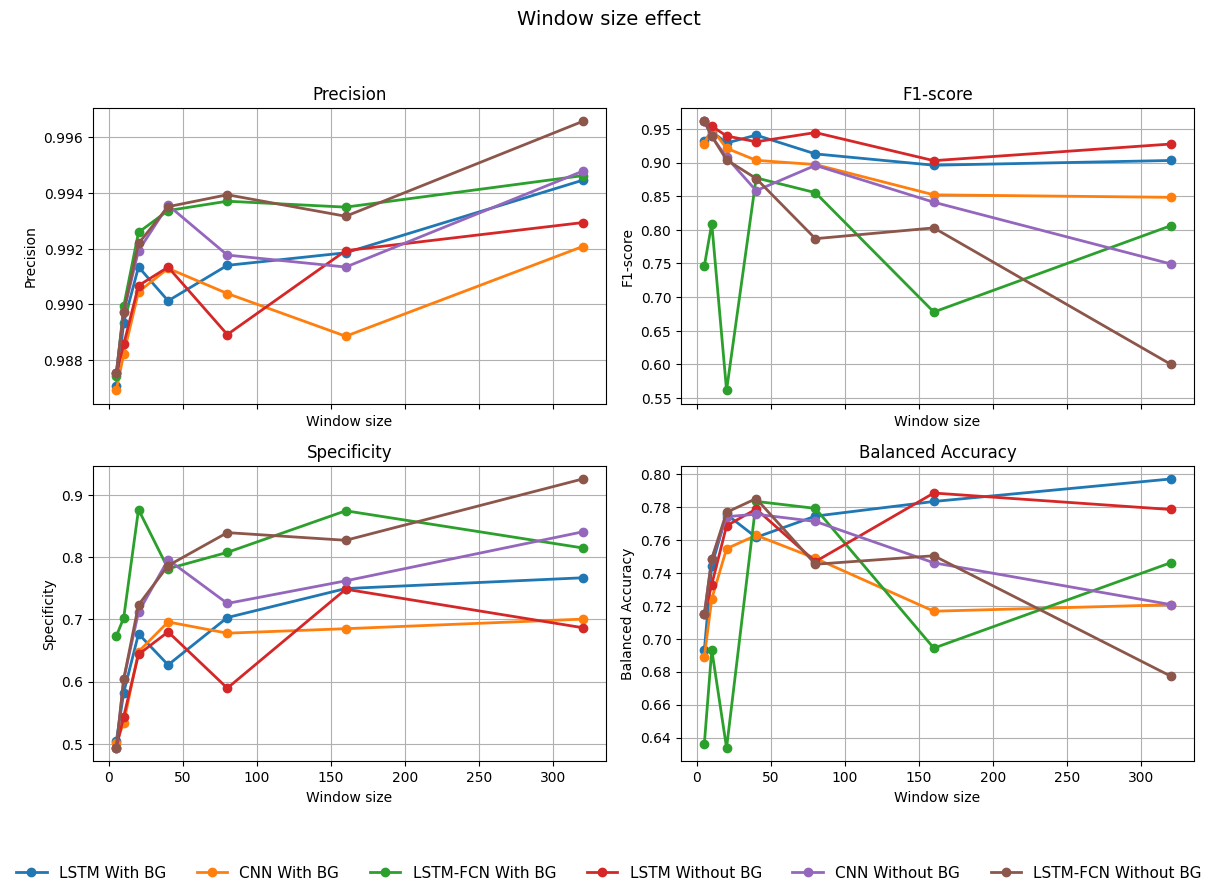}
	\caption{Effect of input window size on validation performance across all evaluated model architectures, with and without background information, for a 5-day prediction horizon. Precision, F1-score, specificity, and balanced accuracy are reported for window sizes ranging from 5 to 320 days. All models were trained using a batch size of $1{,}024$ (established in Section~\ref{sec:batch_size}), output horizon $5$ days, G-Mean loss ($\text{eps} = 10^{-12}$, $\text{use\_log} = \text{True}$), learning rate $10^{-3}$, and early stopping with patience $10$. G-Mean loss has no $\alpha$ or $\gamma$ parameters; no loss function assumptions are made in this sweep.}
	\label{fig:window_comparison}
\end{figure}

Figure~\ref{fig:window_comparison} presents the performance of all evaluated models across window sizes ranging from $5$ to $320$ days, with each step doubling the previous window length. All experiments used a batch size of $1{,}024$ (established in Section~\ref{sec:batch_size}), an output horizon of $5$ days, and the G-Mean loss function. G-Mean loss is parameter-free with respect to class weighting, so no loss function assumptions are made in this sweep. Results are assessed on the validation set.

\begin{table}[ht]
\centering
\begin{tabular}{lcccc}
\hline
\multicolumn{5}{c}{\textbf{Window size = 40}} \\
\hline
Model & Precision & F1-Score & Specificity & Balanced Accuracy \\
\hline
LSTM With BG & 0.990 & 0.941 & 0.627 & 0.761 \\
CNN With BG & 0.991 & 0.903 & 0.696 & 0.763 \\
LSTM-FCN With BG & 0.993 & 0.877 & 0.781 & 0.783 \\
LSTM Without BG & 0.991 & 0.931 & 0.680 & 0.779 \\
CNN Without BG & 0.994 & 0.858 & 0.796 & 0.775 \\
LSTM-FCN Without BG & 0.994 & 0.876 & 0.786 & 0.785 \\
\hline
\multicolumn{5}{c}{\textbf{Window size = 80}} \\
\hline
Model & Precision & F1-Score & Specificity & Balanced Accuracy \\
\hline
LSTM With BG & 0.991 & 0.913 & 0.703 & 0.774 \\
CNN With BG & 0.990 & 0.897 & 0.678 & 0.749 \\
LSTM-FCN With BG & \textbf{0.994} & 0.855 & \textbf{0.807} & 0.779 \\
LSTM Without BG & 0.989 & 0.945 & 0.589 & 0.747 \\
CNN Without BG & 0.992 & 0.896 & 0.725 & 0.771 \\
LSTM-FCN Without BG & \textbf{0.994} & 0.787 & \textbf{0.839} & 0.745 \\
\hline
\multicolumn{5}{c}{\textbf{Window size = 160}} \\
\hline
Model & Precision & F1-Score & Specificity & Balanced Accuracy \\
\hline
LSTM With BG & 0.992 & 0.896 & 0.749 & 0.783 \\
CNN With BG & 0.989 & 0.852 & 0.685 & 0.717 \\
LSTM-FCN With BG & 0.993 & 0.678 & 0.874 & 0.694 \\
LSTM Without BG & 0.992 & 0.903 & 0.748 & 0.788 \\
CNN Without BG & 0.991 & 0.841 & 0.762 & 0.746 \\
LSTM-FCN Without BG & 0.993 & 0.803 & 0.827 & 0.750 \\
\hline
\end{tabular}
\caption{Model performance at window sizes $40$, $80$, and $160$ days across all architectures, with and without background information. All models trained with batch size $1{,}024$, output horizon $5$ days, G-Mean loss (no $\alpha$ or $\gamma$), learning rate $10^{-3}$, early stopping patience $10$, evaluated on the validation set. Bold values indicate the highest specificity configurations at window size $80$.}
\label{tab:window_comparison}
\end{table}

From an architectural perspective, the LSTM-FCN model demonstrates superior performance in terms of precision and specificity, while maintaining competitive F1 score and balanced accuracy, regardless of whether background information is included. Given the emphasis on minimising false positive rates in this study, the LSTM-FCN architecture represents the most suitable choice among the evaluated models.

Across all architectures, models trained with and without background information exhibit broadly similar performance trends as the window size increases. Precision shows a modest upward trend with increasing window length. F1 score generally decreases as the window size grows, with the LSTM-FCN model trained with background information displaying greater variability. Specificity consistently improves with longer window sizes across models, indicating enhanced ability to correctly identify genuinely absent employees. Balanced accuracy increases with window size for the LSTM model throughout the examined range, while the CNN and LSTM-FCN models improve up to approximately $80$ days, after which performance begins to decline.

Table~\ref{tab:window_comparison} compares window sizes of $40$, $80$, and $160$ days. Increasing the window size to $160$ days leads to a notable decline in both F1-score and balanced accuracy for the LSTM-FCN architecture. For a 5-day prediction horizon, the optimal trade-off between performance and stability is achieved with window sizes between $40$ and $80$ days, where models maintain consistently strong performance across all evaluation metrics.

Overall, incorporating staff background information does not significantly improve performance, while the combination of the LSTM-FCN architecture and a window size of $40$ to $80$ days yields the most robust and reliable performance within the proposed TSC framework.

\subsection{Loss function comparison}

Table~\ref{tab:loss_comparison} presents the performance of the LSTM-FCN model on the validation set under a systematic sweep of BFL configurations and G-Mean loss, with window size $80$, output horizon $5$, and batch size $1{,}024$. We include one baseline configuration: BFL with $\alpha = 0.75$ and $\gamma = 2$. We include one baseline configuration: BFL with $\alpha = 0.75$ and $\gamma = 2$, following the standard focal loss parameterisation~\parencite{lin2017focal} and prior imbalanced classification studies~\parencite{bordia2024bonafidelegallens2024shared}. This baseline is included not as an experimental choice but as a reference point representing common practice, to establish how severely the standard parameterisation fails when the majority class is labelled positive. Beyond the baseline, the class-weighting parameter $\alpha$ was varied across four theoretically motivated values derived from the imbalance ratio $\rho = N^+/N^- = 42.31$ (computed from the training data); the analytically balanced value $\alpha^* = 1/(1+\rho) \approx 0.023$, and three values in the near-balanced range $\alpha \in \{0.05, 0.1, 0.2\}$. The focusing parameter $\gamma$ was varied across $\{0, 1, 2, 5\}$. All non-baseline values of $\alpha$ are grounded in $\rho$ alone, as derived in Section~\ref{sec:theoretical_analysis}.

\begin{table}[H]
  \centering
  \begin{tabular}{llcccc}
    \hline
    \textbf{Loss} & \textbf{Config} &
    \textbf{Precision} & \textbf{F1} &
    \textbf{Specificity} & \textbf{Bal.\ Acc.} \\
    \hline
    \multicolumn{6}{l}{\textit{Baseline (standard parameterisation from prior work)}} \\
    BFL & $\alpha=0.75$, $\gamma=2$ & 0.979 & 0.989 & 0.124 & 0.562 \\
    \hline
    \multicolumn{6}{l}{\textit{Theoretically motivated sweep ($\alpha$ derived from $\rho$)}} \\
    BFL & $\alpha^{*}$, $\gamma=0$ & \textbf{0.995} & 0.979 & 0.813 & \textbf{0.888} \\
    BFL & $\alpha^{*}$, $\gamma=1$ & 0.992 & 0.938 & 0.707 & 0.799 \\
    BFL & $\alpha^{*}$, $\gamma=2$ & 0.993 & 0.927 & 0.743 & 0.807 \\
    BFL & $\alpha^{*}$, $\gamma=5$ & 0.992 & 0.915 & 0.725 & 0.787 \\
    \hline
    BFL & $\alpha=0.05$, $\gamma=0$ & 0.991 & 0.960 & 0.648 & 0.789 \\
    BFL & $\alpha=0.05$, $\gamma=1$ & 0.992 & 0.952 & 0.682 & 0.799 \\
    BFL & $\alpha=0.05$, $\gamma=2$ & 0.995 & 0.982 & 0.795 & 0.882 \\
    BFL & $\alpha=0.05$, $\gamma=5$ & 0.987 & 0.962 & 0.516 & 0.727 \\
    \hline
    BFL & $\alpha=0.10$, $\gamma=0$ & 0.992 & 0.983 & 0.672 & 0.823 \\
    BFL & $\alpha=0.10$, $\gamma=1$ & 0.994 & 0.992 & 0.764 & 0.877 \\
    BFL & $\alpha=0.10$, $\gamma=2$ & 0.995 & 0.988 & 0.785 & 0.884 \\
    BFL & $\alpha=0.10$, $\gamma=5$ & 0.984 & 0.982 & 0.371 & 0.675 \\
    \hline
    BFL & $\alpha=0.20$, $\gamma=0$ & 0.985 & 0.987 & 0.401 & 0.695 \\
    BFL & $\alpha=0.20$, $\gamma=1$ & 0.986 & 0.986 & 0.414 & 0.700 \\
    BFL & $\alpha=0.20$, $\gamma=2$ & 0.990 & 0.991 & 0.581 & 0.787 \\
    BFL & $\alpha=0.20$, $\gamma=5$ & 0.987 & 0.985 & 0.472 & 0.728 \\
    \hline
    G-Mean & --- & 0.994 & 0.766 & \textbf{0.844} & 0.734 \\
    \hline
  \end{tabular}
  \caption{Performance of BFL across configurations and G-Mean loss, evaluated on the validation set. Positive = present ($y=1$, majority); negative = absent ($y=0$, minority). Specificity measures the correct identification of absent employees. The baseline row ($\alpha=0.75$, $\gamma=2$) follows the standard focal loss parameterisation~\parencite{lin2017focal}. $\alpha^* = 1/(1+\rho) \approx 0.023$ is the analytically balanced value derived from $\rho = 42.31$ (training data). Bold values indicate the best result in each column.}
    \label{tab:loss_comparison}
\end{table}

There are a few findings from the Table~\ref{tab:loss_comparison}. The standard parameterisation fails for this problem. The baseline configuration ($\alpha = 0.75, \gamma = 2$), following \textcite{lin2017focal}, achieves specificity of only $0.124$. The model collapses to predicting presence for all employees, correctly identifying fewer than one in eight genuinely absent staff. This failure is not incidental; as the theoretical analysis in Section~\ref{sec:theoretical_analysis} shows, $\alpha = 0.75$ generates a gradient ratio of $\rho \times 0.75/0.25 = 126.9$ at initialisation under the present $=$ positive convention, amplifying rather than mitigating majority class dominance. The standard parameterisation was designed for settings where the positive class is the minority; it is systematically miscalibrated when applied to the reverse convention. 

The dominant factor is $\alpha$. Specificity increases monotonically as $\alpha$ decreases from $0.75$ towards $\alpha^*$, consistent with the theoretical prediction that gradient balance requires $\alpha^* = 1/(1+\rho)$. At $\alpha = 0.20$ (gradient ratio $\approx 10.6$ at initialisation), specificity collapses to $0.40--0.58$ regardless of $\gamma$. At $\alpha^*$ (gradient ratio $= 1$), specificity reaches $0.71--0.81$ depending on $\gamma$. When $\alpha$ is correctly calibrated, increasing $\gamma$ reduces both specificity and balanced accuracy consistently. At $\alpha^* = 0.023$, the best configuration is $\gamma = 0$ (no focusing), achieving the highest balanced accuracy at $\alpha^*$. This is consistent with the theoretical analysis; when the gradient is already balanced by $\alpha^*$, the reactive focusing mechanism introduced by $\gamma$ disrupts the equilibrium rather than improving it.

BFL with $\alpha^*$ and $\gamma=0$ is competitive with G-Mean, but requires calibration. The BFL configuration ($\alpha^*, \gamma=0$) achieves specificity $0.813$ and balanced accuracy $0.888$, compared with G-Mean's specificity $0.844$ and balanced accuracy $0.734$. The improvement over the baseline is dramatic; specificity increases from $0.124$ to $0.813$ simply by replacing $\alpha=0.75$ with $\alpha^*=0.023$. Crucially, however, BFL requires knowing $\rho$ in advance to compute $\alpha^*$, while G-Mean's self-correcting gradient structure adapts automatically to the class imbalance without any parameter calibration.

\subsection{Theoretical analysis of the specificity gap}\label{sec:theoretical_analysis}

The results in Table~\ref{tab:loss_comparison} present two striking contrasts. First, the baseline configuration ($\alpha = 0.75$, $\gamma = 2$) achieves specificity of only $0.124$, whilst the analytically derived $\alpha^{*}$ with $\gamma = 0$ achieves $0.813$, a difference of $0.689$ from a single parameter change. Second, G-Mean loss achieves specificity $0.844$ without any parameter calibration at all. We now provide the theoretical justification for both contrasts, showing that they follow directly from the gradient dynamics of BFL and G-Mean loss, with $\rho = 42.31$ computed from the training data as the only fixed quantity. 

\subsubsection{Gradient dynamics of binary focal loss}

For a model with sigmoid output $p = \sigma(z)$ representing the predicted probability of presence ($y=1$), the BFL gradient with respect to the pre-sigmoid logit $z$ is obtained via the chain rule ($\partial p / \partial z = p(1-p)$):

\begin{align}
  \frac{\partial \text{FL}}{\partial z}\bigg|_{y=1\,(\text{present})} &=
  \alpha(1-p)^{\gamma-1}p\bigl[\gamma(1-p)\log p + p - 1\bigr],
  \label{eq:bfl_grad_pos}\\
  \frac{\partial \text{FL}}{\partial z}\bigg|_{y=0\,(\text{absent})} &=
  -(1-\alpha)p^{\gamma-1}(1-p)\bigl[\gamma p \log(1-p) - p\bigr].
  \label{eq:bfl_grad_neg}
\end{align}

Let $g_1$ and $g_0$ denote the signed bracket terms from Equations~\ref{eq:bfl_grad_pos} and~\ref{eq:bfl_grad_neg} respectively:

\begin{align}
  g_1 &= (1-p)^{\gamma-1}p\bigl[\gamma(1-p)\log p + p - 1\bigr] < 0,
  \label{eq:g1}\\
  g_0 &= p^{\gamma-1}(1-p)\bigl[\gamma p\log(1-p) - p\bigr] < 0,
  \label{eq:g0}
\end{align}

\noindent so that the gradient of a single present sample is $\alpha g_1$ (negative, since the loss decreases as $p$ increases towards $1$) and the gradient of a single absent sample is $-(1-\alpha)g_0$ (positive, since the loss increases as $p$ increases away from $0$). The gradient magnitudes are $\alpha |g_1|$ and $(1-\alpha) |g_0|$ respectively. Since $g_1<0$ and $g_0<0$, the ratio $g_1/g_0 > 0$, and $g_1 \approx g_0$ when $p \approx 0.5$ by symmetry of the two loss terms about the decision boundary.

In a training batch of size $B$, the expected counts are $n^+ = B\rho/(1+\rho)$ (present, majority) and $n^- = B/(1+\rho)$ (absent, minority). The total gradient contribution from each class is the product of its sample count and per-sample gradient magnitude. The ratio of total gradient magnitudes from the present class to the absent class is therefore:

\begin{equation}
  \frac{\|\nabla_\text{present}\|}{\|\nabla_\text{absent}\|} \approx
  \underbrace{\rho}_{\substack{\text{count} \\ \text{ratio}}}
  \cdot\,
  \underbrace{\frac{\alpha}{1-\alpha}}_{\substack{\text{weight} \\ \text{ratio}}}
  \cdot\,
  \underbrace{\frac{g_1}{g_0}}_{\substack{\text{focusing} \\ \text{ratio}}}.
  \label{eq:bfl_ratio}
\end{equation}

The batch size $B$ cancels in this ratio, so the result is independent of batch size. Early in training, when $p\approx 0.5$, we have $g_1/g_0 \approx 1$ as established above, so~\eqref{eq:bfl_ratio} reduces to:

\begin{equation}
  \frac{\|\nabla_\text{present}\|}{\|\nabla_\text{absent}\|}
  \;\approx\; \rho \cdot \frac{\alpha}{1-\alpha}.
  \label{eq:bfl_ratio_early}
\end{equation}

The focusing parameter $\gamma$ has dropped out entirely. This is an exact result at $p=-0.5$; the gradient imbalance at initialisation is determined solely by $\rho$ (from the data) and $\alpha$ (a parameter choice). The gradient balanced value of $\alpha$ is derived in Proposition~\ref{prop:bfl_alpha} below as $\alpha^* = 1/(a+\rho)$. For any $\alpha > \alpha^*$, the present class dominates the gradient signal from the first update, and $\gamma$ cannot correct this because it has no effect at $p=0.5$.

To see why the baseline fails, substitute $\rho = 42.31$ and $\alpha=0.75$;

\begin{equation}
    \frac{\|\nabla_\text{present}\|}{\|\nabla_\text{absent}\|}
  \;\approx\; 42.31 \times \frac{0.75}{0.25} \approx 126.9.
  \label{eq:bfl_ratio_baseline}
\end{equation}

The present class contributes approximately $127$ times more gradient signal than the absent class at initialisation. This is a compounding of two effects: the class count ratio $\rho = 42.31$ and the weight ratio $\alpha/(1-\alpha) = 3$, the latter arising because $\alpha = 0.75$ places \emph{greater} weight on the positive (present, majority) class rather than the minority. The standard parameterisation was designed for settings where $y=1$ is the minority class; applied in reverse, it actively amplifies majority-class dominance. The analytically correct value $\alpha^{*} = 1/(1+\rho) \approx 0.023$ brings the ratio to exactly $1$, as proved in Proposition~\ref{prop:bfl_alpha}.

The focusing parameter $\gamma$ only suppresses present-class gradient contributions \emph{after} those samples become well-classified (i.e.\ after $p \to 1$ for present samples and $(1-p)^\gamma \to 0$). This mechanism is inherently \emph{reactive}: it can only reduce majority-class dominance after the model has already learned to confidently predict presence, a state that itself requires majority-class gradients to dominate the early training trajectory.

\begin{proposition}\label{prop:bfl_alpha}
    For BFL to achieve balanced gradient contributions from present and absent classes $\|\nabla_\text{present}\| \approx \|\nabla_\text{absent}\|$ during early training under imbalance ratio $\rho = N^+/N^-$, the class-weighting parameter must satisfy:

    \begin{equation}
        \alpha^* = \frac{1}{1+\rho}.
    \end{equation}

    For any $\alpha > \alpha^*$, the present class dominates the gradient signal at initialisation by a factor of $\rho\alpha/(1-\alpha) > 1$. At $\rho = 42.31$ (from the training data), gradient balance requires $\alpha^* \approx 0.023$.
    
\end{proposition}

\begin{proof}
    Setting~\eqref{eq:bfl_ratio_early} equal to unity and solving for $\alpha$ gives $1 = \rho\alpha/(1-\alpha)$, hence $\alpha = 1/(1+\rho)$.
\end{proof}

Proposition~\ref{prop:bfl_alpha} has a useful corollary: for a practitioner targeting a gradient ratio of $r$, the required $\alpha$ is:

\begin{equation}
    \alpha = \frac{r}{\rho+r}. 
\end{equation}

Since $\rho$ is known exactly from the data before training begins, $\alpha$ can be set analytically without grid search. The sweep in Table~\ref{tab:loss_comparison} confirms this empirically; specificity increases monotonically as $\alpha$ decreases towards $\alpha^*$.

\subsubsection{Self-correcting gradient structure of G-Mean loss} \label{self-correcting G-Mean}

The gradients of the G-Mean loss (Equation~\ref{eq:gmean}) with respect to individual predictions are obtained by differentiating $\mathcal{L}_{\text{GM}} = -\frac{1}{2}\log(\widehat{\text{TPR}} \cdot \widehat{\text{TNR}})$ through the soft TPR and TNR approximation (Equations~\ref{eq:soft_tpr} and~\ref{eq:soft_tnr}):

\begin{align}
  \frac{\partial \mathcal{L}_{\text{GM}}}{\partial p_i}
  \bigg|_{y_i=1\,(\text{present})}
  &= -\frac{1}{2} \cdot \frac{1}{\widehat{\text{TPR}}}
     \cdot \frac{\partial \widehat{\text{TPR}}}{\partial p_i}
   = -\frac{1}{2} \cdot \frac{1}{\widehat{\text{TPR}}}
     \cdot \frac{1}{N^{+}}
   = -\frac{1}{2N^{+}\,\widehat{\text{TPR}}},
  \label{eq:gmean_grad_pos} \\
  \frac{\partial \mathcal{L}_{\text{GM}}}{\partial p_j}
  \bigg|_{y_j=0\,(\text{absent})}
  &= -\frac{1}{2} \cdot \frac{1}{\widehat{\text{TNR}}}
     \cdot \frac{\partial \widehat{\text{TNR}}}{\partial p_j}
   = -\frac{1}{2} \cdot \frac{1}{\widehat{\text{TNR}}}
     \cdot \frac{-1}{N^{-}}
   = -\frac{1}{2N^{-}\,\widehat{\text{TNR}}}.
  \label{eq:gmean_grad_neg}
\end{align}

Summing over all present and absent samples in the batch, the $N^{\pm}$ terms cancel, yielding the ratio of total gradient contributions from present to absent employees:

\begin{equation}
  \frac{\|\nabla_\text{present, total}\|}{\|\nabla_\text{absent, total}\|}
  = \frac{N^{+} \cdot \dfrac{1}{2N^{+}\,\widehat{\text{TPR}}}}
         {N^{-} \cdot \dfrac{1}{2N^{-}\,\widehat{\text{TNR}}}}
  = \frac{\widehat{\text{TNR}}}{\widehat{\text{TPR}}}.
  \label{eq:gmean_ratio}
\end{equation}

\begin{proposition}\label{prop:gmean}
    Under G-Mean loss, the ratio of total gradient contributions from present to absent employees equals $\widehat{TNR}/\widehat{TPR}$, which is independent of the imbalance ratio $\rho$. Furthermore, the gradient structure is self-correcting against majority class collapse. When the model collapses towards predicting present for all employees ($\widehat{TPR} \to 1, \widehat{TNR} \to 0)$,

    \begin{equation}
        \frac{\|\nabla_\text{present, total}\|}{\|\nabla_\text{absent, total}\|}
    = \frac{\widehat{\text{TNR}}}{\widehat{\text{TPR}}} \to 0,
    \end{equation}

so the absent class gradient dominates, providing a restoring force that prevents sustained collapse and preserves specificity regardless of $\rho$.
\end{proposition}.

\begin{proof}
  The expression $\widehat{\text{TNR}}/\widehat{\text{TPR}}$ in equation~\eqref{eq:gmean_ratio} contains no dependence on $N^{+}$, $N^{-}$, or $\rho$. As $\widehat{\text{TPR}} \to 1$ and $\widehat{\text{TNR}} \to 0$ under majority-class collapse, the ratio tends to $0$, so the normalised gradient from each absent sample diverges relative to each present sample, amplifying the absent-class signal.
\end{proof}

In contrast to BFL, G-Mean loss exhibits a self-correcting gradient structure: the closer the model is to predicting present for everyone, the stronger the relative gradient signal from the absent (minority) class. This property is entirely absent from BFL, whose gradient ratio~\eqref{eq:bfl_ratio_early} depends directly on $\rho$ and is compounded by the $\alpha/(1-\alpha)$ weighting term, providing no proactive protection against the early-training collapse observed in Table~\ref{tab:loss_comparison}.

\subsubsection{Implications for loss function selection}

The analyses above provide a principled explanation for all findings in Table~\ref{tab:loss_comparison}. The baseline configuration ($\alpha = 0.75$, $\gamma = 2$) generates a gradient ratio of approximately $126.9$ at initialisation (equation~\eqref{eq:bfl_ratio_baseline}), driving the model to predict presence for almost all employees and yielding specificity $0.124$. This failure is structural, not incidental: the standard parameterisation is miscalibrated for any setting where the positive class is the majority. For any $\alpha > \alpha^{*} = 1/(1+\rho)$, majority-class gradient dominance occurs at initialisation and $\gamma$ cannot prevent it, because $\gamma$ has exactly zero effect
at $p = 0.5$.

The sweep confirms these predictions quantitatively. Specificity increases monotonically as $\alpha$ decreases from $0.75$ towards $\alpha^{*}$, and $\gamma = 0$ is optimal at $\alpha^{*}$ because the gradient is already balanced and the focusing mechanism only introduces reactive disruption. The practical guidance is;

\begin{itemize}
  \item If $\rho$ is known and stable, BFL with $\alpha^{*} = 1/(1+\rho)$ and $\gamma = 0$ is a viable option, achieving specificity $0.813$ and balanced accuracy $0.888$.
  \item The standard parameterisation ($\alpha = 0.75$, $\gamma = 2$) should not be used when the positive class is the majority, regardless of $\rho$. The gradient ratio formula $\rho\alpha/(1-\alpha)$ quantifies the miscalibration exactly.
  \item G-Mean loss achieves specificity $0.844$ without any parameter calibration, making it the preferable default when $\rho$ is unknown or when the practitioner wishes to avoid manual calibration entirely.
\end{itemize}

\subsubsection{Empirical verification of gradient dynamics}

\begin{figure}[H]
  \centering
  \includegraphics[width=1\textwidth]{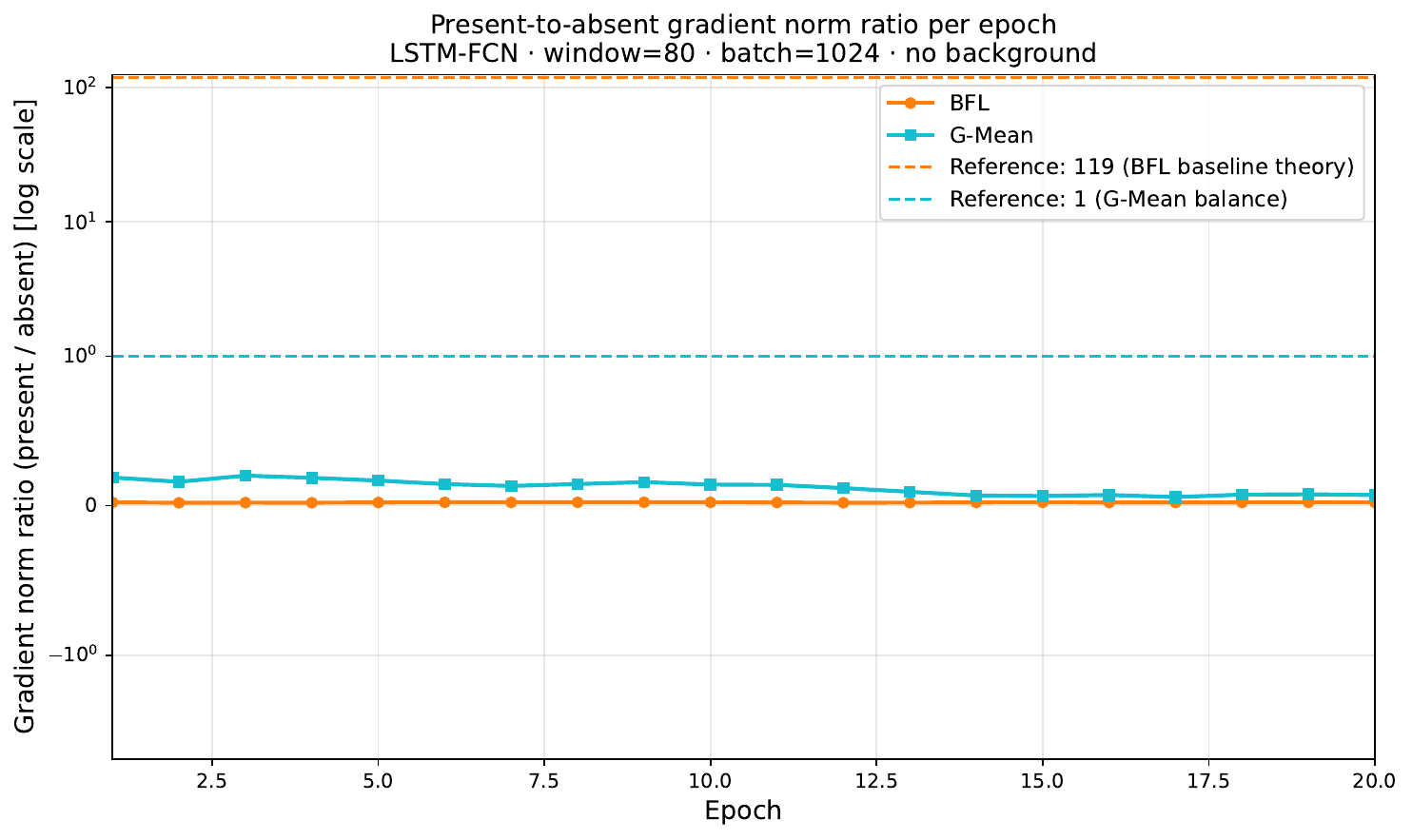}
  \caption{Present-to-absent gradient norm ratio (log scale) per epoch under BFL ($\alpha=0.75$, $\gamma=2$, baseline parameterisation) and G-Mean loss, computed on the LSTM-FCN model with window size $80$, batch size $1{,}024$, without background information. Both models are initialised from the same random seed. The orange colour dashed reference line shows the theoretical prediction of $\rho\alpha/(1-\alpha) \approx 119$ for the baseline BFL configuration at $p=0.5$, where $\rho=39.5$ is computed from the windowed training sample after sliding window extraction and user-level splitting. The full dataset gives $\rho=42.31$ and ratio $\approx 126.9$; the difference reflects the sliding window extraction procedure and does not affect any conclusion. The cyan colour dashed line indicates gradient balance (ration $=1$). Both models operate well below the theoretical prediction throughout training, confirming that bath normalisation in the FCN branch substantially moderates the scalar logit-level gradient imbalance. BFL remains near zero throughout, while G-MEan maintains a slightly higher ratio, consistent with the self-correcting property of Proposition~\ref{prop:gmean}}
  \label{fig:gradient_ratio}
\end{figure}

\begin{figure}[H]
  \centering
  \includegraphics[width=1\textwidth]{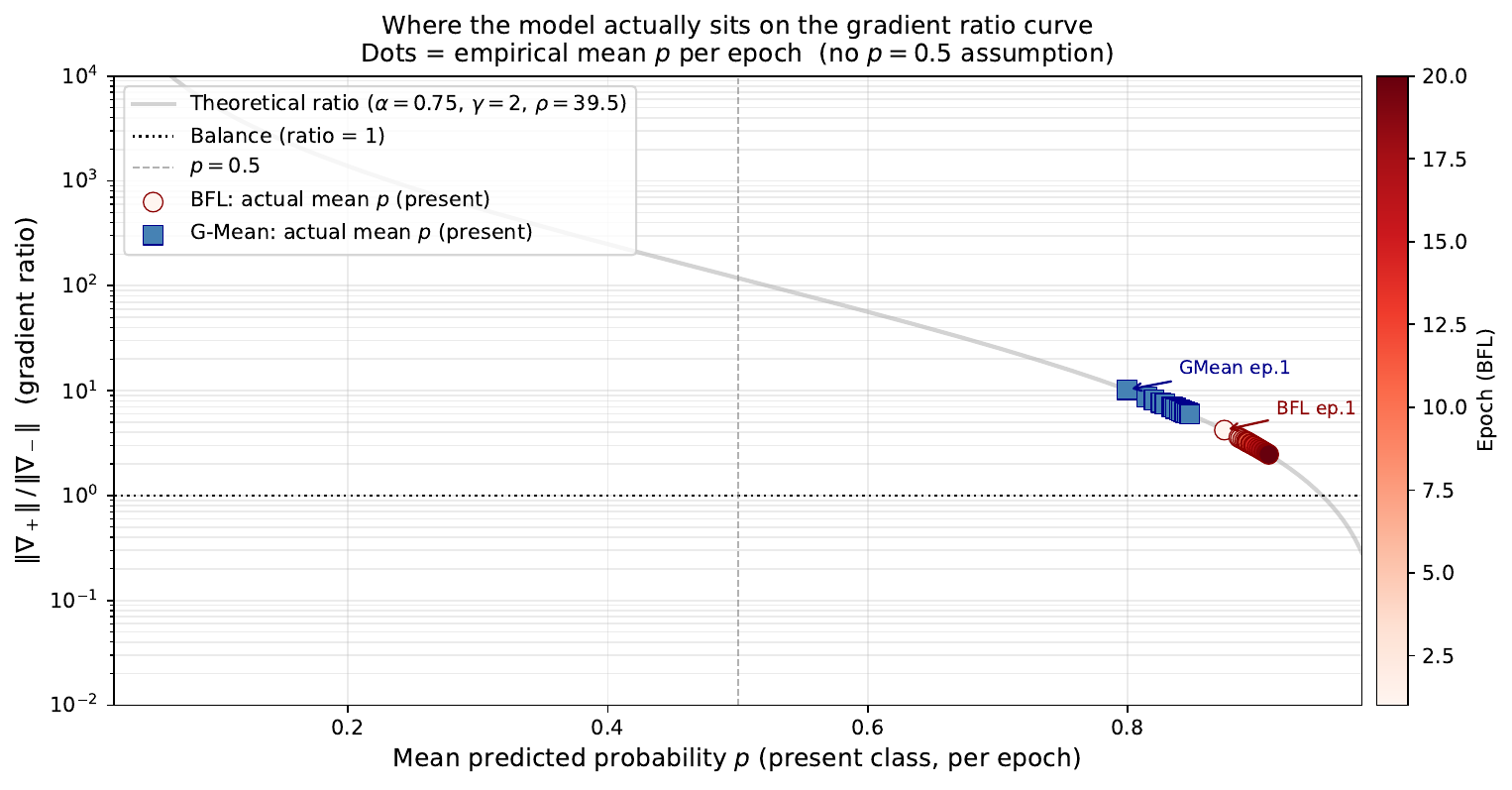}
  \caption{Empirical mean predicted probability $p$ per epoch for the present class, plotted against the theoretical gradient ratio curve $\rho\alpha/(1-\alpha) \cdot g_1/g_0$ for the baseline BFL configuration ($\alpha=0.75$, $\gamma=2$, $\rho=39.5$ from the training partition). Dots represent the actual mean $p$ at each training epoch; colour indicates epoch number (light to dark). Both BFL and G-Mean operate in the range $p \approx 0.8$--$0.87$ throughout training, far to the right of the $p=0.5$ initialisation point. This confirms that the model crosses $p = 0.5$ within the first few training batches, before the epoch-level mean is recorded. At $p > 0.5$, the theoretical gradient ratio is already well below the $p=0.5$ prediction, and both models experience present-class gradient suppression rather than dominance at the epoch level.}
  \label{fig:p_on_ratio_curve}
\end{figure}

Figure~\ref{fig:gradient_ratio} presents the present-to absent gradient norm ratio over the first $20$ training epochs under both loss functions, measured on the LSTM- FCN model under identical initialisations. The BFL run uses the baseline parameterisation ($\alpha=0.75, \gamma=2)$ to demonstrate the gradient dynamics that produce the specificity collapse of $0.124$ observed in Table~\ref{tab:loss_comparison}. Three phenomena are visible, each with a theoretical interpretation. 

\begin{itemize}
    \item[I)] Batch normalisation moderates the theoretical prediction. The scalar logit-level analysis of Section~\ref{sec:theoretical_analysis} predicts aBFL gradient ratio of $\rho\alpha/(1-\alpha) \approx 119$ at $p=0.5$ (training partition $\rho = 39.5$). The observed ratios for both models are substantially below this throughout all $20$ epochs. This is attributable to the batch normalisation layers in the FCN branch, which standardise feature activations across all samples in the batch jointly:

    \begin{equation}
        \hat{x} = \frac{x-\mu_{\text{batch}}}{\sqrt{\sigma_{batch}^2} + \varepsilon},
    \end{equation}

    acting as an adaptive gradient smoother during backpropagation. The scalar analysis should therefore be understood as an upper bound on the gradient imbalance in architectures that include batch normalisation.

    \item[II)] Both models operate far from $p=0.5$. Figure~\ref{fig:p_on_ratio_curve}confirms both models already have man $p \approx 0.8$--$0.87$ for present samples. The model crosses $p=0.5$ within the first few training batches, before the epoch-level mean is recorded. At $p>0.5$, the focusing term $1-p^\gamma$ suppresses present-class gradients aggressively, driving the gradient ratio well below $1$ for both models. This is not a correction of the class imbalance; it is reactive over-suppression occurring because the model has already committed to predicting presence for most employees.

    \item[III)] BFL collapses specificity while G-Mean maintains it. Despite both models showing similar gradient ratios at the epoch level, their specificity outcomes differ dramatically; BFL achieves specificity $0.124$ while G-Mean achieves $0.844$. The p-distribution figures (Figures~\ref{fig:p_mean_trajectory}--\ref{fig:p_hist_gmean}) explains the reasons. Under BFL, the absent-class mean $p$ starts at $0.775$ at epoch $1$ and declines to $0.62$ by epoch $20$, remaining above the decision boundary of $0.5$ throughout. Under G-Mean, the absent-class mean $p$ drops immediately to $0.22$ at epoch $1$ and stabilises at $0.185$ by epoch $20$. G-Mean's self-correcting gradient structure (Proposition~\ref{prop:gmean}) drives absent-class predictions below $0.5$ from the very first epoch, while BFL cannot achieve this despite the apparent gradient  suppression.

\end{itemize}

\subsubsection{Theoretical analysis of class weighting parameter ($\alpha$) and focusing parameter ($\gamma$)}

\begin{figure}[H]
  \centering
  \includegraphics[width=\textwidth]{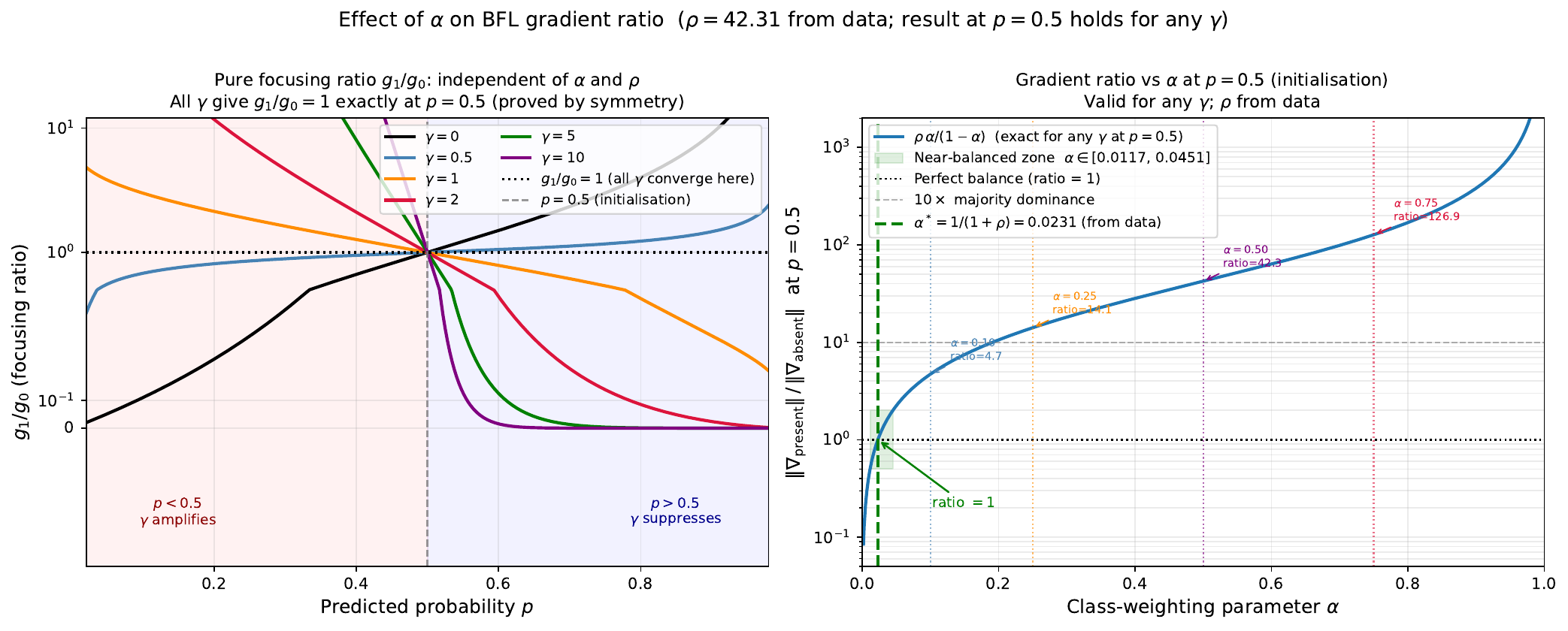}
  \caption{Effect of $\alpha$ on the BFL gradient ratio ($\rho = 42.31$ computed from data; result at $p = 0.5$ holds for any $\gamma$). \textit{Left:} Pure focusing ratio $g_1/g_0$ across $p$ for all $\gamma$ values, independent of $\alpha$ and $\rho$. All curves converge to $g_1/g_0 = 1$ exactly at $p = 0.5$, proving by symmetry that $\gamma$ has zero effect at initialisation. When $p < 0.5$ (model wrong about present employees), $\gamma$ amplifies majority dominance; when $p > 0.5$ (model correct), $\gamma$ suppresses it reactively. \textit{Right:} Early-training gradient ratio $\rho\alpha/(1-\alpha)$ as a function of $\alpha$, valid for any $\gamma$ at $p = 0.5$. The balanced value $\alpha^* = 1/(1+\rho) \approx 0.023$ is marked in green. The near-balanced zone (ratio within a factor of 2 of unity) is $\alpha \in [0.012, 0.045]$, shown as a shaded region. Reference values $\alpha \in \{0.10, 0.25, 0.50, 0.75\}$ are annotated with their corresponding gradient ratios, demonstrating that commonly used values produce substantial majority-class dominance at initialisation.}
  \label{fig:alpha_analysis}
\end{figure}

\begin{figure}[H]
  \centering
  \includegraphics[width=\textwidth]{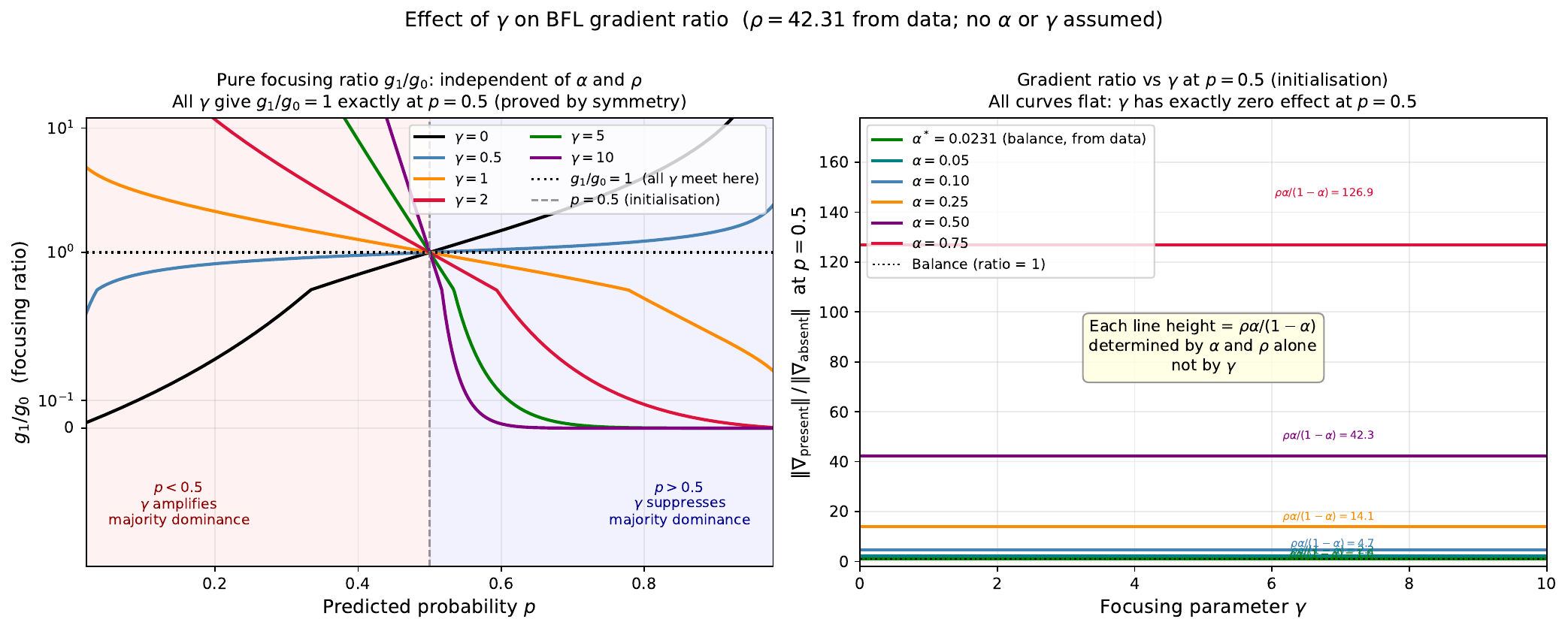}
  \caption{Effect of $\gamma$ on the BFL gradient ratio ($\rho = 42.31$ computed from data; no $\alpha$ or $\gamma$ assumed). \textit{Left:} Pure focusing ratio $g_1/g_0$ across $p$ for all $\gamma$ values, identical to Figure~\ref{fig:alpha_analysis} (left). \textit{Right:} Gradient ratio vs $\gamma$ at $p = 0.5$ (initialisation) for a range of $\alpha$ values. Every curve is perfectly flat: $\gamma$ has exactly zero effect at $p = 0.5$ regardless of $\alpha$ or $\rho$. The height of each flat line equals $\rho\alpha/(1-\alpha)$, determined entirely by $\alpha$ and $\rho$ alone. This confirms that no choice of $\gamma$ can correct the gradient imbalance at initialisation; $\gamma$ only modulates the ratio reactively once $p$ departs from $0.5$.}
  \label{fig:gamma_analysis}
\end{figure}

Figures~\ref{fig:alpha_analysis} and~\ref{fig:gamma_analysis} provide a comprehensive theoretical analysis of the roles of $\alpha$ and $\gamma$ in BFL gradient dynamics, grounded solely in $\rho$ computed from the data. Three conclusions follow directly. First, $\gamma$ has exactly zero effect on the gradient ratio at $p = 0.5$. This is not an approximation but an exact result: the focusing ratio $g_1/g_0 = 1$ at $p = 0.5$ for any $\gamma$, as proved by symmetry of the BFL loss about $p = 0.5$. The gradient ratio at initialisation is therefore $\rho\alpha/(1-\alpha)$, a function of $\alpha$ and $\rho$ alone. Second, $\alpha$ determines the gradient imbalance at initialisation. At $\rho = 42.31$, the balanced value is $\alpha^* = 1/(1+\rho) \approx 0.023$, far below the conventional recommendation of $\alpha = 0.75$. The formula $\alpha = r/(\rho + r)$ gives the required $\alpha$ for any target gradient ratio $R$, computed analytically from $\rho$ without grid search. Third, $\gamma$ only acts reactively. Once training begins and $p$ moves away from $0.5$, $\gamma$ modulates the focusing ratio: it amplifies majority dominance when $p < 0.5$ (model wrong) and suppresses it when $p > 0.5$ (model correct). But by the time $p >
0.5$, the initial gradient imbalance has already driven the model towards predicting presence for all employees. The focusing mechanism cannot undo this damage.

\subsubsection{Empirical analysis of predicted probability distributions}\label{sec:empirical_analysis}

\begin{figure}[H]
  \centering
  \includegraphics[width=\textwidth]{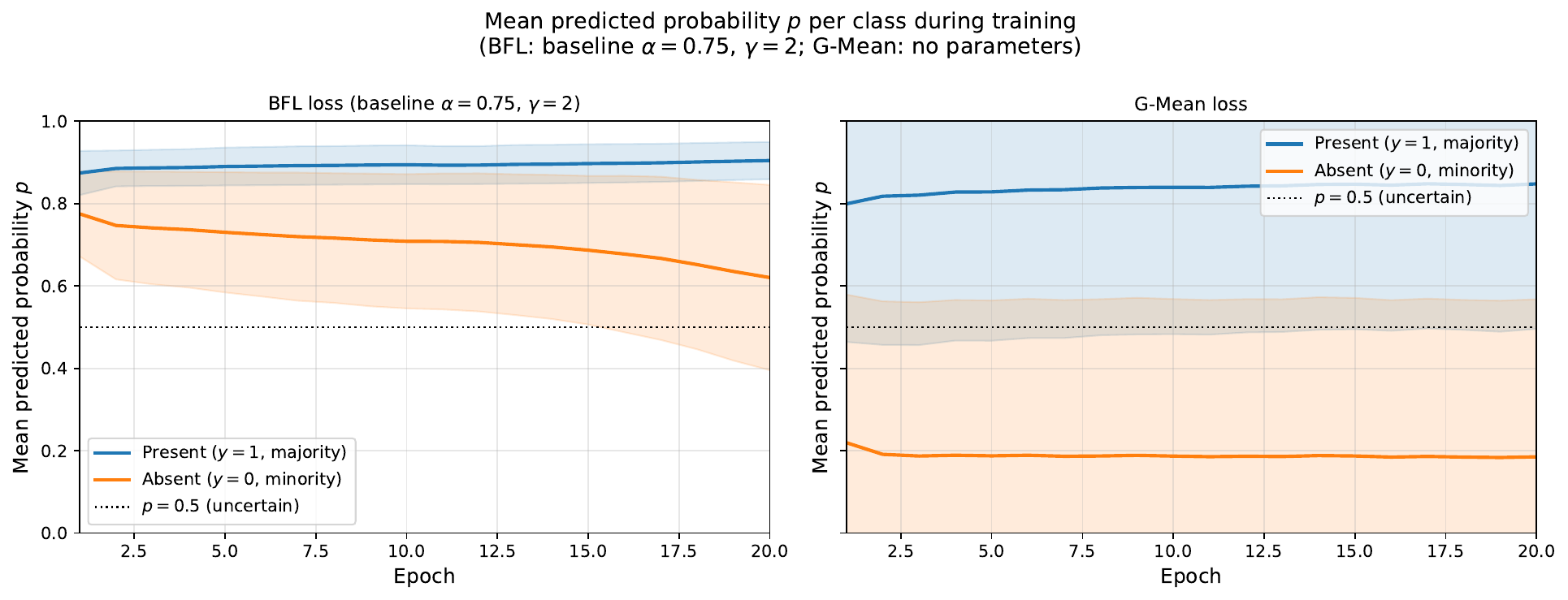}
  \caption{Mean predicted probability $p = \sigma(z)$ per class during training across 20 epochs under BFL loss (left, baseline parameterisation $\alpha=0.75$, $\gamma=2$) and G-Mean loss (right), measured on the LSTM-FCN model. Shaded bands indicate $\pm 1$ standard deviation. Under BFL, the absent-class mean $p$ starts at $0.775$ at epoch $1$ and declines to $0.62$ by epoch $20$, remaining above the decision boundary of $0.5$ throughout training, indicating that absent employees are consistently predicted as present. Under G-Mean loss, the absent-class mean $p$ drops to $0.220$ by epoch $1$ and stabilises at approximately $0.185$ thereafter, reflecting immediate and stable class separation.}
  \label{fig:p_mean_trajectory}
\end{figure}

\begin{figure}[H]
  \centering
  \includegraphics[width=\textwidth]{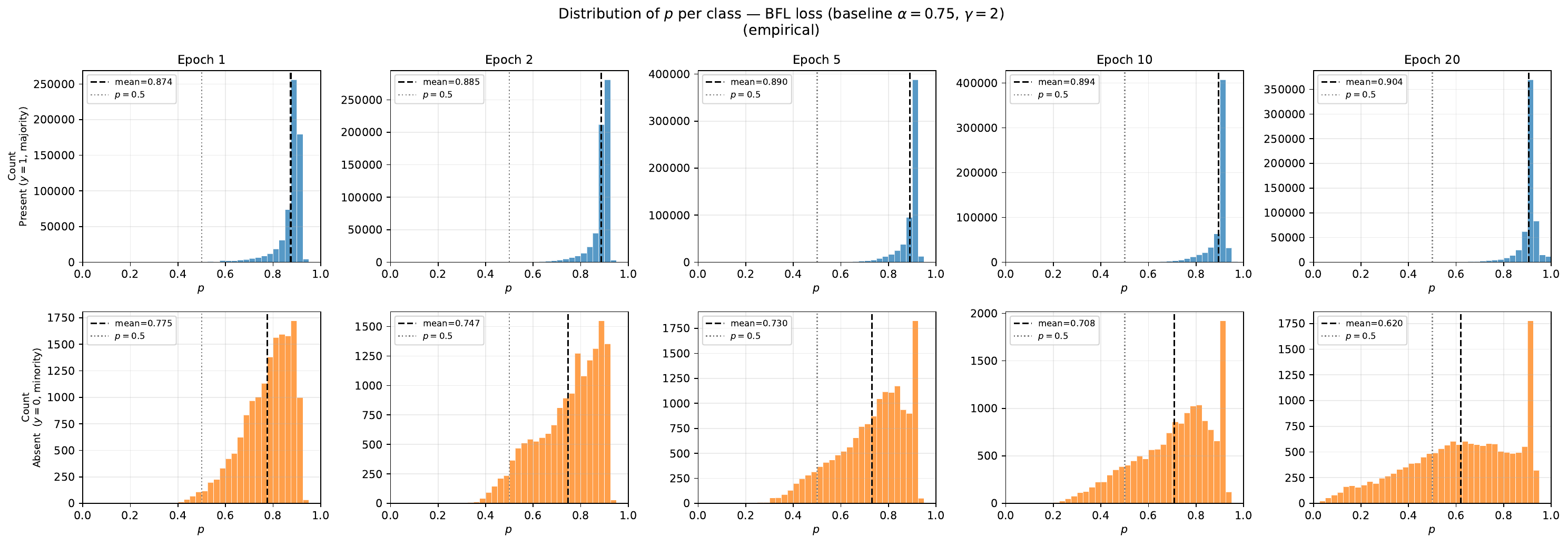}
  \caption{Distribution of $p$ per class at epochs $1$, $2$, $5$,
    $10$, and $20$ under BFL loss with baseline parameterisation
    ($\alpha=0.75$, $\gamma=2$). These figures demonstrate the
    predicted probability distributions that produce the specificity
    collapse of $0.124$ observed in
    Table~\ref{tab:loss_comparison}. \textit{Top row (present,
    $y=1$, majority):} The distribution is sharply concentrated near
    $p \approx 0.87$--$0.91$ from epoch $1$ onwards, indicating the
    model predicts presence with high confidence for present employees
    throughout training. \textit{Bottom row (absent, $y=0$,
    minority):} The distribution is broadly spread across $[0.5,
    1.0]$ at all epochs, with the mean declining from $0.775$ at epoch
    $1$ to $0.620$ by epoch $20$. Absent employees are predicted as present more often
    than not throughout the entire training run.}
  \label{fig:p_hist_bfl}
\end{figure}

\begin{figure}[H]
  \centering
  \includegraphics[width=\textwidth]{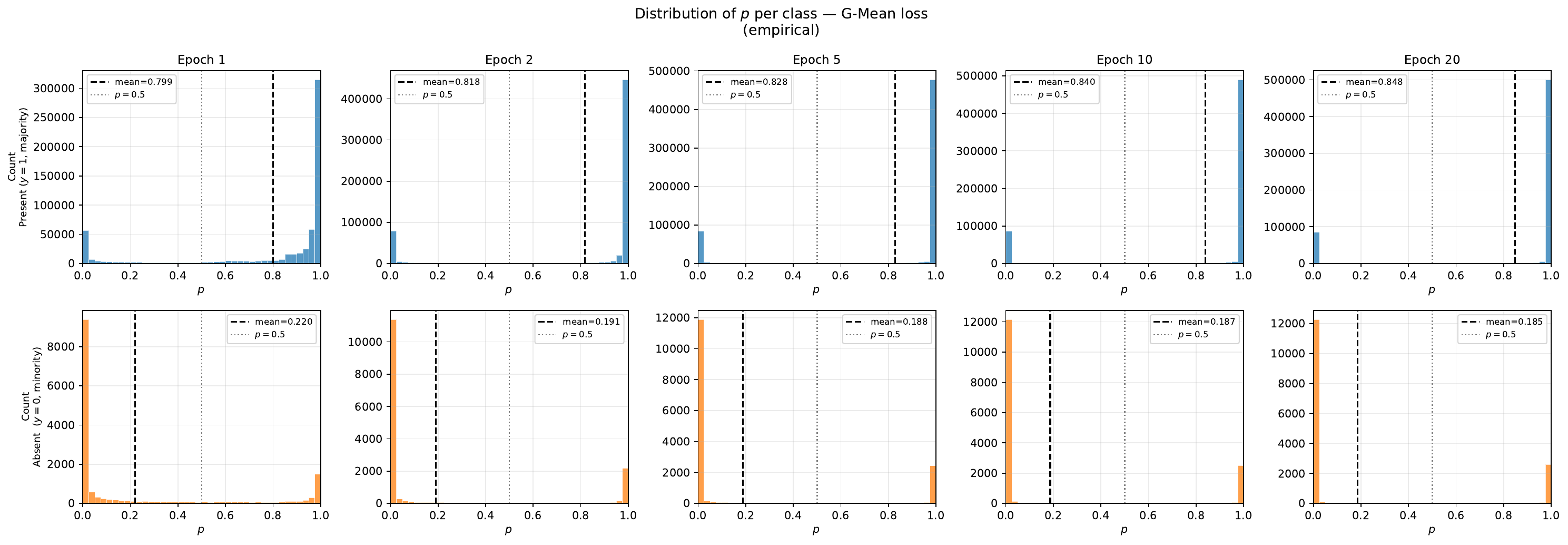}
  \caption{Distribution of $p$ per class at epochs $1$, $2$, $5$,
    $10$, and $20$ under G-Mean loss (empirical; no assumptions about
    $\alpha$ or $\gamma$). \textit{Top row (present, $y=1$, majority):}
    Bimodal distribution with a large spike near $p \approx 1$ and a
    smaller spike near $p \approx 0$, reflecting confident predictions
    in both directions. \textit{Bottom row (absent, $y=0$, minority):}
    The distribution collapses to a sharp spike near $p \approx 0$ from
    epoch $2$ onwards, with the mean stabilising at approximately
    $0.185$--$0.220$. Absent employees are correctly identified as
    absent with high confidence throughout training, consistent with
    the high specificity of $0.844$ reported in
    Table~\ref{tab:loss_comparison}.}
  \label{fig:p_hist_gmean}
\end{figure}

Figures~\ref{fig:p_mean_trajectory}--\ref{fig:p_hist_gmean} present the empirical distribution of $p = \sigma(z)$ per class during training under both loss functions. The BFL figures use the baseline parameterisation ($\alpha = 0.75$, $\gamma = 2$) to provide empirical evidence for the specificity collapse analysed theoretically in Section~\ref{sec:theoretical_analysis}. The G-Mean figures use no $\alpha$ or $\gamma$ parameters.

The most important observation is that neither model passes through $p = 0.5$ at the epoch level. By the time the first epoch completes, the mean $p$ for both classes is already substantially above $0.5$ under both loss functions. This confirms that the model crossed the initialisation point within the first few training batches — before the epoch-level mean was recorded — consistent with the theoretical prediction of extreme gradient imbalance at $p = 0.5$.

The key contrast between the two loss functions lies in the trajectory of the absent class. Under BFL, the absent-class $p$ starts at $0.775$ at epoch $1$ and declines to $0.620$ by epoch $20$, remaining above $0.5$ throughout all $20$ epochs. Under G-Mean loss, the
absent-class $p$ drops immediately to $0.220$ at epoch $1$ and remains stable at $0.185$ by epoch $20$. The self-correcting gradient structure of G-Mean (Proposition~\ref{prop:gmean}) prevents the absent-class $p$ from being swept upward alongside the present class,
which is precisely what occurs under BFL.

\subsection{Evaluation on the test set}\label{sec:evaluation_testset}

\begin{table}[H]
	\centering
	\caption{Test set performance under different window size}
	\label{tab:test_performance_days}
	\begin{tabular}{ccccc}
		\hline
		\textbf{Window size} & \textbf{Precision} & \textbf{F1 Score} & \textbf{Specificity} & \textbf{Balanced Accuracy} \\
		\hline
		40  & 0.993 & 0.883 & 0.785 & 0.790 \\
		80  & 0.994 & 0.859 & 0.815 & 0.786 \\
		160 & 0.993 & 0.653 & 0.877 & 0.681 \\
		\hline
	\end{tabular}
\end{table}

Test set evaluation is conducted using the LSTM-FCN model trained with background information, using a batch size of $1{,}024$, a window size from $40$ to $160$, an output horizon of $5$, and the G-Mean loss function. The balanced accuracy for window sizes between $40$ and $80$ days is approximately $80\%$, which is consistent with results reported in prior studies on the same dataset, although those studies primarily used standard accuracy as the evaluation metric. The test set performance is consistent with the results observed on the validation set, indicating that the model generalises reasonably well to previously unseen data. This consistency suggests that the learned temporal patterns are not over-fitted to the training or validation samples, supporting the robustness of the proposed approach.

\section{Conclusion and future research}\label{sec:conclusion}

This paper proposes a Time Series Classification framework for individual-level staff absenteeism prediction and provides both theoretical and empirical foundations for its practical application. The central methodological contribution is the reformulation of absenteeism prediction as a TSC problem at the individual level, explicitly separating historical attendance sequences from future absence labels to enable genuinely proactive workforce decision support, in contrast to existing approaches that model already-realised outcomes or discard individual-level sequential structure.

A key theoretical contribution of this work is the gradient analysis of loss function behaviour under severe class imbalance, with $\rho = N^{+}/N^{-} = 42.31$ computed directly from the training data as the only fixed quantity. We derive analytically that the BFL
gradient ratio at initialisation is $\rho\alpha/(1-\alpha)$, and that the class weight required for gradient balance is $\alpha^{*} = 1/(1+\rho) \approx 0.023$. This is an exact result requiring no assumption about $\alpha$ or $\gamma$: it follows from $\rho$ alone. The focusing parameter $\gamma$ has exactly zero effect at initialisation (proved by symmetry at $p = 0.5$), and acts only reactively once $p$ departs from $0.5$ during training. The systematic sweep across $\alpha \in \{\alpha^{*}, 0.05, 0.10, 0.20\}$ and $\gamma \in \{0, 1, 2, 5\}$ confirms these theoretical predictions: specificity increases monotonically as $\alpha$ decreases towards $\alpha^{*}$, and $\gamma = 0$ produces the best results at $\alpha^{*}$, confirming that the focusing mechanism is unnecessary and harmful when $\alpha$ is correctly calibrated. The best BFL configuration ($\alpha^{*}, \gamma = 0$) achieves specificity $0.813$ and balanced accuracy $0.888$. G-Mean loss achieves specificity $0.844$ and balanced accuracy $0.734$, making a different trade-off between the two classes. The critical practical distinction is that BFL requires knowing $\rho$ in advance to compute $\alpha^{*}$, while G-Mean's self-correcting gradient structure, whose ratio $\widehat{\text{TNR}}/\widehat{\text{TPR}}$ adapts continuously to the current model state, achieves competitive performance without any parameter calibration. These findings generalise to any severely imbalanced binary classification task in which the majority class is labelled positive, including healthcare monitoring, fraud detection, and fault detection, where $\rho$ is known from the data and the practitioner must choose between a calibrated BFL and a calibration-free G-Mean.

Empirically, we establish that the LSTM-FCN architecture achieves the most consistent precision and specificity; that model performance stabilises for batch sizes of at least $64$, consistent with the requirement that absent-class samples appear reliably in each batch for the G-Mean self-correcting mechanism to function; and that window sizes between $40$ and $80$ days yield the most balanced performance for a 5-day prediction horizon, achieving balanced accuracy of approximately $80\%$ on held-out test data. Incorporating staff background information does not significantly improve predictive performance, suggesting that sequential attendance history alone carries sufficient signal for individual-level prediction within this framework.

The principal limitation of this study is the use of simulated longitudinal attendance data, necessitated by the absence of publicly available individual-level records. The simulation is rigorously calibrated to the distributional properties of the UCI \emph{Absenteeism at Work} benchmark, and all code and data are publicly released to support reproducibility and future comparative evaluation. Validation on real organisational attendance data remains the most important direction for future research, and constitutes a prerequisite for deployment in production workforce management systems. Additional future directions include extending the framework to multivariate attendance sequences incorporating contextual signals such as team composition and shift patterns, and investigating whether the gradient ratio formula $\rho\alpha/(1-\alpha)$ provides reliable practical guidance for loss function selection across other severely imbalanced sequential prediction domains.

\section*{Data and code availability}

All code, data and review results are available at:

\url{https://github.com/nonstopronald/Time-Series-Classification-for-Staff-Absenteeism-Prediction}


\printbibliography


\end{document}